\documentclass[lettersize,journal]{IEEEtran}
\usepackage{amsmath,amsfonts}
\usepackage{algorithmic}
\usepackage{array}
\usepackage{textcomp}
\usepackage{stfloats}
\usepackage{url}
\usepackage{verbatim}
\usepackage{graphicx}
\usepackage{xcolor}
\usepackage{soul}
\usepackage[hidelinks]{hyperref}
\usepackage[utf8]{inputenc}
\usepackage{amsmath}
\usepackage{subcaption}
\usepackage{amsthm}
\usepackage{booktabs}
\usepackage{algorithm}
\usepackage[switch]{lineno}
\usepackage{cite}
\usepackage{bbding}

\newtheorem{theorem}{Theorem}
\newtheorem{definition}{Definition} 

\def\BibTeX{{\rm B\kern-.05em{\sc i\kern-.025em b}\kern-.08em
    T\kern-.1667em\lower.7ex\hbox{E}\kern-.125emX}}
\usepackage{balance}
\begin{document}
\title{Adaptive $k$ Nearest Neighbors Classifier via Granular Ball Computing}
\author{Xiaoyu Lian, Shuyin Xia*, Hongxuan He, Lifeng Shen, Guoyin Wang, Xinbo Gao
\thanks{X. Lian, S. Xia, H. He, L. Shen and X. Gao are with the Chongqing Key Laboratory of Computational Intelligence, Key Laboratory of Cyberspace Big Data Intelligent Security, Ministry of Education, Sichuan-Chongqing Co-construction Key Laboratory of Digital Economy Intelligence and Key Laboratory of Big Data Intelligent Computing, Chongqing University of Posts and Telecommunications, 400065, Chongqing, China. E-mail: lianxiaoyu724@qq.com, xiasy@cqupt.edu.cn, hhxuan0726@163.com, shenlf@cqupt.edu.cn, gaoxb@cqupt.edu.cn.}

\thanks{G. Wang are with the College of Computer and Information Science, the National Center for Applied Mathematics in Chongqing, Chongqing Normal University, Chongqing 401331, China. E-mail: wanggy@cqnu.edu.cn.}
}


\maketitle

\begin{abstract}
The $k$-Nearest Neighbor~(KNN) algorithm is widely used across various tasks. 
The selection of the $k$ value is a key issue because it significantly impacts performance.
In this paper, an adaptive and efficient KNN approach via granular-ball computing is proposed. 
The method consists of two stages. \textcolor{black}{In the training stage, the dataset is first coarsely partitioned to reduce the complexity of data distributions within a granular ball, and then the Fisher criterion is introduced to control ball splitting and stopping, yielding a multi-granularity granular ball representation.
In the prediction stage, the nearest granular ball is first located through a weighted distance mechanism, and an adaptive neighborhood is then constructed around the test sample. The effective $k$ value is dynamically determined by the actual number of samples contained in this neighborhood. The neighborhood induced by the nearest granular ball provides more stable local group information, thereby improving robustness against noise and local perturbations.} Experimental results demonstrate that the proposed method outperforms existing KNN variants across multiple datasets in terms of both accuracy and efficiency. 
The code has been open-sourced for reproducibility: https://github.com/lianxiaoyu724/Adaptive-GBKNN.

\end{abstract}


\section{Introduction}
\IEEEPARstart{T}{he} $k$-nearest neighbors~(KNN) algorithm is a classic supervised learning algorithm~\cite{soucy2001simple}. 
It computes Euclidean distances from the test instance to all training samples, subsequently identifying the $k$ nearest neighbors for final classification through majority voting on their class labels.
It is widely applied in applications such as speech recognition~\cite{paliwal1983application}, pattern recognition~\cite{zhang2017high,wani1handwritten,ali2022intelligent,zhang2024hierarchical,ab2021efficientnet}, classification and clustering tasks~\cite{heo2018distance, samet2007k, zhang2004fast, kim1986fast, goin1984classification}.
However, the algorithm still has two limitations. In high-dimensional and large-scale datasets, the cost of distance computation and sorting increases sharply, leading to unacceptable computational complexity~\cite{ukey2023survey}. Moreover, the choice of $k$ affects classification accuracy: a small $k$ may cause overfitting, while a large $k$ tends to include noisy samples, thereby compromising the robustness of the model.

\begin{figure}[t] 
	\centering 
	\includegraphics[width=0.45\textwidth]{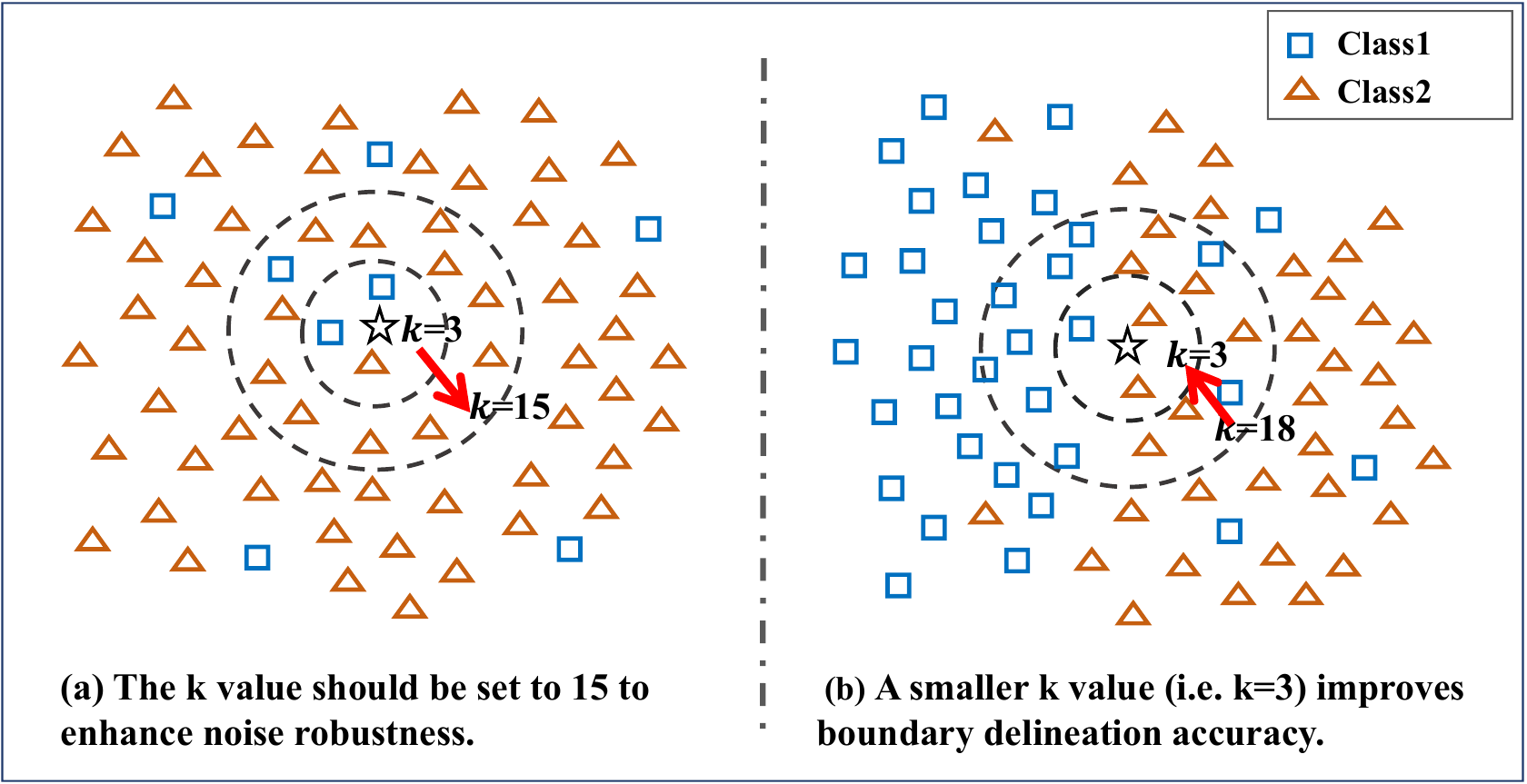} 
	\caption{The effectiveness of k selection on KNN classification, in which the test sample should be classified to Class 2.} 
	\label{knn}
\end{figure}

The selection of the $k$ value of the KNN algorithm plays a critical role in determining classification accuracy. As shown in Fig.~\ref{knn}(a), within the local distribution of Class 2, a few Class 1 samples can be regarded as noise. Setting $k=3$ may cause the test sample to be misclassified as ``Class 1.'' Therefore, in the presence of noise, a larger $k$~(e.g., $k=15$) should be selected to enhance model robustness. Conversely, when the test sample lies near the decision boundary, as shown in Fig.~\ref{knn}(b), a smaller $k$ is preferable to better capture local structural characteristics and improve classification precision. Currently, improved KNN algorithms that estimate the optimal $k$ value~\cite{ghosh2006optimum} or use distance metrics~\cite{domeniconi2002locally,bhattacharya2012affinity} have been widely applied in various research fields. The concept of neighborhood search has become a significant development direction for KNN. In 2007, Yiu et al. extended the traditional reverse nearest neighbor query by introducing the concept of reverse $k$-nearest neighbor~(RKNN) search~\cite{yiu2007reverse}. However, like traditional KNN, RKNN is also susceptible to the choice of $k$. The core challenge lies in effectively detecting the neighborhood structure for different datasets, which is especially difficult when the data characteristics are unknown.

Regarding this challenge, the natural neighbor of an interesting idea was proposed, which has triggered extensive research~\cite{pourbahrami2022geometric}.
Natural neighbors represent a scale-free neighbor relationship that better reflects the true characteristics of the data. Two data points, $x$ and $y$, form a natural neighbor relationship if $x$ considers $y$ as a neighbor and $y$ likewise considers $x$ as a neighbor~\cite{zhu2016natural}. 
If each object in a dataset can find its natural neighbor, the dataset reaches a stable state. 
\textcolor{black}{The natural neighborhood provides an adaptive alternative to the preset $k$ in traditional kNN. Nevertheless, methods built on natural neighborhoods tend to rely heavily on the local geometric structure of data points~\cite{yang2024gnan}.}
When noise or outliers are present in the data, these anomalous points can significantly alter the neighborhood structure, leading to a decline in the stability and accuracy of the classification results. Especially when the data exhibits high-dimensional and complex distributions~(e.g., non-convex shapes or multimodal distributions), the discriminative power of distance metrics decreases, and natural neighbor relationships may fail to accurately capture the boundary features between classes, negatively impacting classification performance.

\textcolor{black}{In recent years, some studies have introduced granular-ball computing into KNN classification. By replacing original sample points with granular balls and making decisions based on the nearest granular ball, these methods improve classification efficiency and robustness to a certain extent~\cite{xia2019granular,xia2022efficient}. 
However, existing methods mainly rely on purity-related functions during granular ball generation, while purity is insufficient to fully characterize the geometric structure of samples within a granular ball. In the classification stage, they usually make decisions directly based on a single nearest granular ball, lacking explicit modeling of the local neighborhood of the query sample. 
Although a study has further combined granular-ball computing with natural neighborhood methods (MGKNN)~\cite{xie2024mgnr}, their granular ball partitioning processes still mainly rely on unsupervised mechanisms and cannot fully exploit class supervision information. 
Therefore, how to adaptively generate better supervised granular ball representations and, based on them, construct an adaptive neighborhood decision mechanism for query samples remains a worthwhile problem for further study.}

To address the above limitations, this paper proposes an efficient and adaptive \textcolor{black}{granular ball} KNN~(GBKNN) algorithm based on granular-ball computing. The algorithm initially simplifies the data distribution within each \textcolor{black}{granular balls} by splitting the original dataset containing $n$ samples into $\sqrt{n}$ \textcolor{black}{granular balls}. 
\textcolor{black}{The Fisher criterion is jointly introduced to control the adaptive generation of granular balls. 
Subsequently, an adaptive local neighborhood is constructed for the test sample according to its spatial relationship with the nearest granular ball, thereby adaptively determining the effective $k$ value.} Specifically, the contributions of this study include:
\begin{itemize}
    \item \textcolor{black}{A highly efficient and fully adaptive granular ball generation method is proposed. The method first partitions the data at a coarse granularity to reduce the complexity of data distribution in a ball, and then uses the Fisher criterion to guide ball splitting, thereby establishing an adaptive generation framework of granular balls that range from coarse to fine.}
    \item \textcolor{black}{A KNN method that constructs adaptive neighborhood relationships based on granular ball structures is proposed. Instead of classifying directly by the nearest granular ball, the method uses the nearest granular ball as a reference to build a local adaptive neighborhood for the test sample, and adaptively determines $k$ by the actual number of samples contained in that neighborhood.}
    \item \textcolor{black}{Experimental results on multiple benchmark datasets show that GBKNN achieves superior overall performance in terms of classification accuracy, robustness, and time efficiency compared with several competing methods.}
\end{itemize}

The remainder of this paper is organized as follows: Section~\ref{sec:related work} provides an overview of the evolution of KNN techniques, along with a discussion of relevant studies on granular-ball computing. In Section~\ref{sec3}, a fully adaptive method for generating \textcolor{black}{granular balls} is introduced, followed by the design of an adaptive GBKNN algorithm built upon it. Section~\ref{sec:experiment} reports the experimental results and delivers a comprehensive analysis of the algorithm's performance. Finally, Section~\ref{sec6} concludes the paper and highlights promising directions for future exploration.

\section{Related Work}\label{sec:related work}

\subsection{KNN Classification Method}
Neighborhood-based KNN algorithms have been extensively developed across various domains by integrating diverse strategies to overcome the limitations of traditional KNN~\cite{huang2023lightweight,9835174,gong2022joint}. 
These developments can be broadly categorized into three methodological directions. 

First, some methods aim to reconstruct neighborhood structures to better reflect the true data manifold.
To address the rigidity of fixed nearest-neighbor definitions, Yiu et al.~\cite{yiu2007reverse} introduced the reverse KNN (RKNN) model, leveraging reverse nearest neighbor search to enhance query efficiency. Similarly, Brito et al.~\cite{brito1997connectivity} proposed the mutual $k$-nearest neighbor (MKNN) model, which constructs mutual nearest neighbor graphs to improve robustness and reduce computational complexity for large datasets. Further, Zhu et al.~\cite{zhu2016natural} developed the natural neighbor (NaNEKNN) method, which eliminates the need for manually choosing the $k$ parameter by identifying neighbors based on mutual proximity in the dataset. These methods redefine how neighborhoods are formed, providing a foundation for more adaptive KNN variants.


Second, adaptive strategies are employed to determine neighborhood sizes based on local density or distribution. Mullick et al.~\cite{mullick2018adaptive} introduced the adaptive KNN (AdaKNN), which utilizes local density and distribution patterns, integrated with neural networks, to improve classification performance. Similarly, OneStepKNN~\cite{zhang2021KNN}, parameterless KNN (PLKNN)\cite{jodas2022pl}, and Dr.k-NN\cite{zhu2022distributionally} dynamically infer the optimal number of neighbors based on data characteristics, rather than relying on a pre-specified $k$ value. These methods effectively alleviate the sensitivity to $k$ and enhance KNN’s adaptability to complex data distributions.

Third, at the representation level, coarse-grained or structured approximations are used to enhance efficiency without significantly compromising accuracy. To address this, Ayyad et al.~\cite{ayyad2019gene} proposed a strategy that defines circular weighted regions around test samples, incorporating SMKNN and LMKNN variants to boost classification accuracy, stability, and efficiency in gene expression data. More recently, Xie et al.~\cite{xie2024mgnr} proposed the MGKNN algorithm, which leverages granular-ball computing to construct multi-granularity ball representations and neighborhood relationships while dynamically determining the $k$-value, thereby significantly improving the classification accuracy of KNN. However, this method still relies on the traditional KNN algorithm and does not fully utilize label information during the \textcolor{black}{granular ball} generation process, leaving room for further improvement.

\subsection{Granular-ball Computing}
Granular-ball computing is an important model in granular computing, whose core idea is to adaptively construct granular balls of varying sizes to cover the sample space, thereby driving the learning process through granular ball-based representations.
In terms of efficiency, the number of granular balls is significantly smaller than the number of samples, thus reducing computational overhead; in terms of robustness, the coarse-granularity nature of granular balls makes them less sensitive to sample noise; in terms of interpretability, the topological structure of granular balls provides multi-granularity descriptions. 
\textcolor{black}{For supervised datasets, granular ball generation is usually set to full coverage, i.e., the coverage ratio is 1, while ball compactness is typically ensured by the average radius. Given a dataset $D=\{x_i \mid  i=1,2,…,n\}$, a commonly used supervised granular ball representation model is defined as:
\begin{equation}
    \begin{split}
     &\min_{c_i,r_i,m}  m  \\
     & s.t. \ \ quality(GB_i) \ge \phi \left ({x_j}\right ),i=1,2,...,m,j=1,2,...,n.
    \end{split}
\end{equation}
where $m$ is the number of granular balls, and $c_i$ and $r_i$ denote the center and radius of the $i$-th granular ball, respectively. The quality of each granular ball, $quality(GB_i)$, is regulated by the constraint function $\phi_{x_j}$, thereby preventing the representation from degenerating into a few low-quality large balls. Hence, the key problem in granular ball generation is how to cover the sample space with as few granular balls as possible while maintaining high representation quality under the given constraints.}

\textcolor{black}{This model usually involves the joint optimization of the number, centers, radii, and quality constraints of granular balls, which leads to a highly non-convex optimization problem. Directly obtaining its global optimum is often computationally intractable. Therefore, inspired by the ``global-first'' strategy, this model is efficiently optimized and approximately solved through a coarse-to-fine granular ball splitting process with quality control.}
Initially, \textcolor{black}{granular ball} generation utilized the 2-means algorithm~\cite{xia2019granular}. The purity threshold needs to be optimized within a specific range to adapt to different datasets and address label noise, but this process is computationally expensive. To address this challenge, Xia et al. proposed a novel \textcolor{black}{granular ball} generation method by controlling the weighted purity sum of sub-\textcolor{black}{granular balls} and accelerating the splitting process using a division method~\cite{xia2022efficient}. However, this method still requires improvements in stability and efficiency. Subsequently, Xie et al.~\cite{xie2024gbg++} introduced the GBG~(GBG++) method, based on an attention mechanism, which significantly improves efficiency and robustness while achieving stability. As a novel approach, granular-ball computing has inspired various theoretical methods in machine learning, including \textcolor{black}{granular ball} classifiers~\cite{xia2024gbsvm}, \textcolor{black}{granular ball} clustering methods~\cite{xie2024w}, \textcolor{black}{granular ball} neural networks~\cite{xia2025graph}, \textcolor{black}{granular ball} rough sets~\cite{xia2023gbrs}, and \textcolor{black}{granular ball} reinforcement learning~\cite{liu2024unlock}. 
Meanwhile, granular-ball computing enhances efficiency and interpretability, bringing benefits that foster broader research and application potential~\cite{kuai2024never}.

\begin{figure*}[!ht]
    \centering
    \begin{subfigure}[b]{0.22\textwidth}
        \centering
        \includegraphics[width=\textwidth]{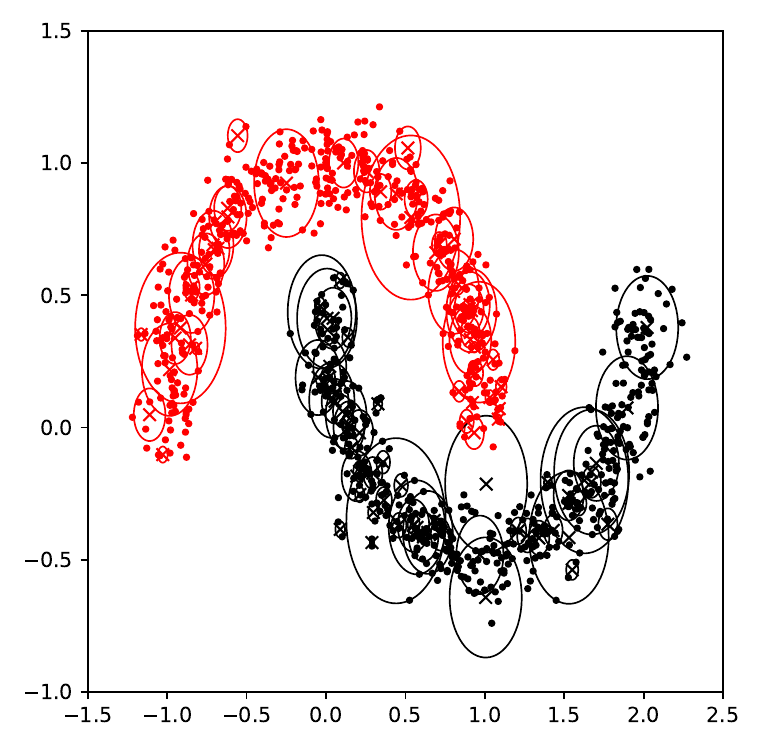}
        \caption{GBG++ on the `Two Moons' dataset.}
        \label{fig:image1}
    \end{subfigure}
    \hfill
    \begin{subfigure}[b]{0.22\textwidth}
        \centering
        \includegraphics[width=\textwidth]{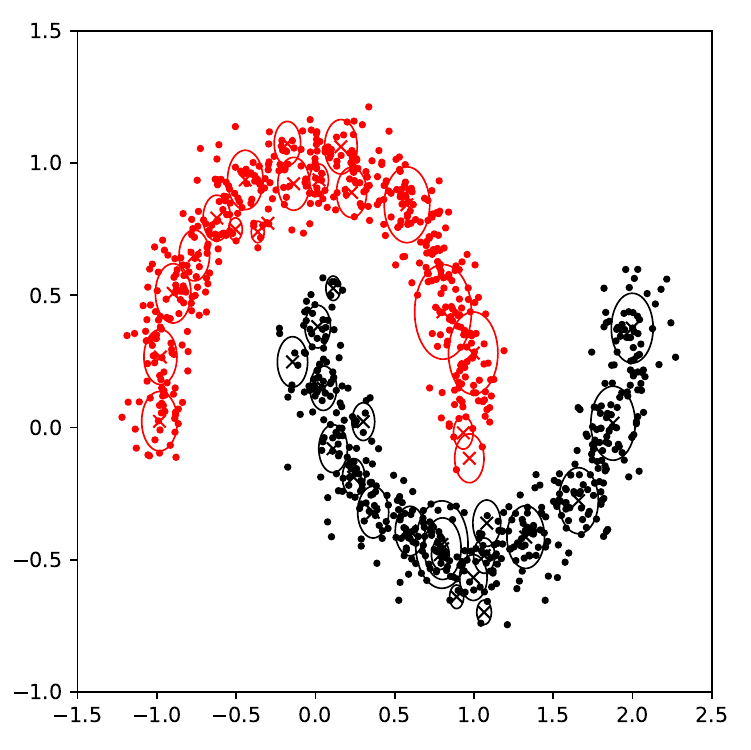}
        \caption{GBG++~(initial $\sqrt{n}$ split) on the `Two Moons' dataset.}
        \label{fig:image12}
    \end{subfigure}
    \hfill
        \begin{subfigure}[b]{0.22\textwidth}
        \centering
        \includegraphics[width=\textwidth]{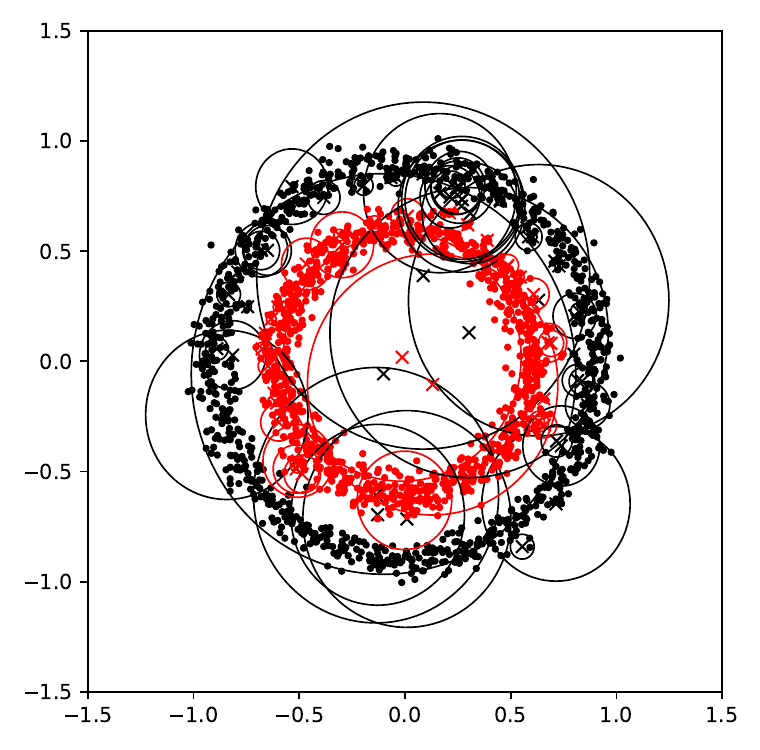}
        \caption{GBG++ on the `Circles' dataset.}
        \label{fig2}
    \end{subfigure}
    \hfill
    \begin{subfigure}[b]{0.22\textwidth}
        \centering
        \includegraphics[width=\textwidth]{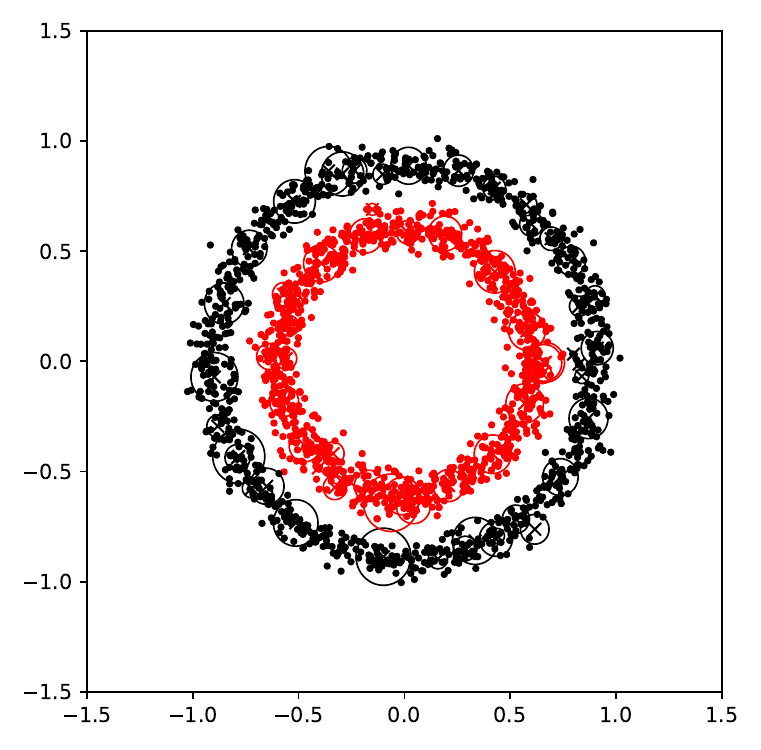}
        \caption{GBG++~(initial $\sqrt{n}$ split) on the `Circles' dataset.}
        \label{fig:image22}
    \end{subfigure}
    \caption{Comparison between GBG++ and GBG++ with initial $\sqrt{n}$ granular ball splitting on two non-linear datasets.}
    \label{sqrt(n)}
\end{figure*}

\section{Preliminaries}
Granular-ball computing takes granular balls as input units for learning. A \textcolor{black}{granular ball} is formally defined as follows~\cite{xia2019granular}.
\begin{definition}
Given a dataset $D=\{x_i, i=1,2,...,n\}$, a \textcolor{black}{granular ball} $GB$ split from $D$ is a spherical structure in the data space, represented by
\begin{equation}
	GB = \{ x_i | x_i \in D,  \Delta \textcolor{black}{\left( x_i - c  \right)  \le  r  \}},
\end{equation}
where \textcolor{black}{$\Delta \left( x_i - c  \right)$} denotes the distance between a sample $x_i$ and the center $c$. In this study, we adopt the Euclidean distance for computation. The center $c$ and radius $r$ of the \textcolor{black}{granular ball} are computed as follows:
\begin{equation}
	c = \frac{1 }{\left | GB \right |}  \sum_{x_i \in GB }  x_i, \ r = \frac{1}{\left | GB \right | }\sum_{x_i \in GB} \Delta \left( x_i-c \right),
\end{equation}
where $| \cdot |$ denotes the cardinality of the set, i.e., the number of elements in the \textcolor{black}{granular ball}.
\end{definition}
In classification problems, each \textcolor{black}{granular ball}, similar to a sample point, is assigned a class label, as defined below.
\begin{definition}\label{label}
Let $GB$ be a \textcolor{black}{granular ball}, and let $L = \{ l^j, j = 1, 2, \ldots, s \}$ denote the set of distinct class labels among its contained samples. 
The label of $GB$, denoted by $l_{GB}$, is defined as the label with the highest frequency among its samples:
\begin{equation}
	l_{GB} = \arg\max_{l^j \in L} \left| \{ x_i \in GB \mid l_{x_i} = l^j \} \right|,
\end{equation}
where $l_{x_i}$ is the label of sample $x_i$.
\end{definition}
According to Definition \ref{label}, the label of a \textcolor{black}{granular ball} is determined by the majority class of its contained samples, thus providing robustness to noise. The quality of a \textcolor{black}{granular ball} can be measured by its purity, defined as follows.
\begin{definition}
The purity of a \textcolor{black}{granular ball} $GB$, denoted by~$P$, measures the proportion of samples in $GB$ that share the \textcolor{black}{granular ball's} label $l_{GB}$:
\begin{equation}
	P = \frac{ \left| \{ x_i \in GB \mid l_{x_i} = l_{GB} \} \right| }{ \left| GB \right| }.
\end{equation}
\end{definition}

In subsequent work, Xie et al.~\cite{xie2024gbg++} introduced a splitting-based mechanism for \textcolor{black}{granular ball} generation, which begins with the entire dataset treated as a single \textcolor{black}{granular ball}. Specifically, the entire dataset is initially regarded as a \textcolor{black}{granular ball}, denoted by $GB$, which fails to meet the required quality criterion. The splitting process proceeds as follows:

Step 1: Compute the majority class label $l^{\ast }$ within the \textcolor{black}{granular ball} $GB$.
\begin{equation}\label{GBG++1}
l^{\ast } =arg \max_{l^{j }\in L} \sum_{i=1}^{n} \mathbb{I}(l_x = l^j),
\end{equation}
where $\mathbb{I}(\cdot)$ denotes the indicator function.

Step 2: Calculate the corresponding center $c_{l^\ast }$ and radius~$r_{l^\ast}$ based on the majority class samples.
To characterize the majority class region within the granular ball $GB$, first extract the subset of samples belonging to the majority class $l^\ast$ and compute its size:
\begin{equation}\label{GBG++2}
    GB_{l^\ast } = \left \{x_i \in GB \mid l_{x_i} = l^\ast  \right \}, \quad  
    N_{l^\ast } = \left | GB_{l^\ast } \right |.
\end{equation}
Then, calculate the center $c_{l^\ast }$ and radius $r_{l^\ast }$ of the majority class samples:
\begin{equation}\label{GBG++22}
    c_{l^\ast } = \frac{1}{N_{l^\ast }} \sum_{x_i \in GB_{l^\ast }}x_i, \quad  
    r_{l^\ast } = \frac{1}{N_{l^\ast }} \sum_{x_i \in GB_{l^\ast }} \Delta \left( x_i - c_{l^\ast }\right ).
\end{equation}
Here, $\Delta(\cdot)$ denotes the Euclidean distance.

Step 3: Construct a sub-\textcolor{black}{granular ball} using points within the computed radius.
\begin{equation}\label{GBG++3}
GB_k = \left \{ x_i \in GB \mid  \Delta \left( x_i-c_{l^\ast}  \right ) \le r_{l^\ast} \right \} .
\end{equation}
\vspace{-1em}
 
The iterative execution of Steps 1 to 3 continues until the number of remaining samples equals the number of class labels. This method has significantly improved the efficiency of \textcolor{black}{granular ball} generation. However, it exhibits limitations when dealing with complex nonlinear data structures, such as `Two Moons' or `Circles' datasets.

Specifically, the GBG++ method constructs each \textcolor{black}{granular ball} by calculating the center and radius based solely on the majority class samples, resulting in an approximately spherical shape. This geometric assumption implicitly presumes isotropic distributions, making it inadequate for fitting elongated or curved data structures. As a result, \textcolor{black}{granular balls} may mistakenly include structurally dissimilar samples or exclude similar ones, leading to structural errors. As illustrated in Fig.~\ref{sqrt(n)}(a) and Fig.~\ref{sqrt(n)}(c), such mismatches hinder the model’s ability to represent high-dimensional, non-uniform data spaces effectively. 
To address this issue, we propose generating $\sqrt{n}$ \textcolor{black}{granular balls} from the global dataset to enable a more appropriate partitioning at the global level, simplifying the internal distribution within each ball and mitigating errors caused by limited local fitting capacity. Furthermore, existing methods typically require repeated tuning of the purity threshold to accommodate different datasets and label noise, which not only relies heavily on manual effort but also incurs considerable computational cost.

\begin{figure*}[!ht]
	\centering
	\includegraphics[width=0.85\textwidth]{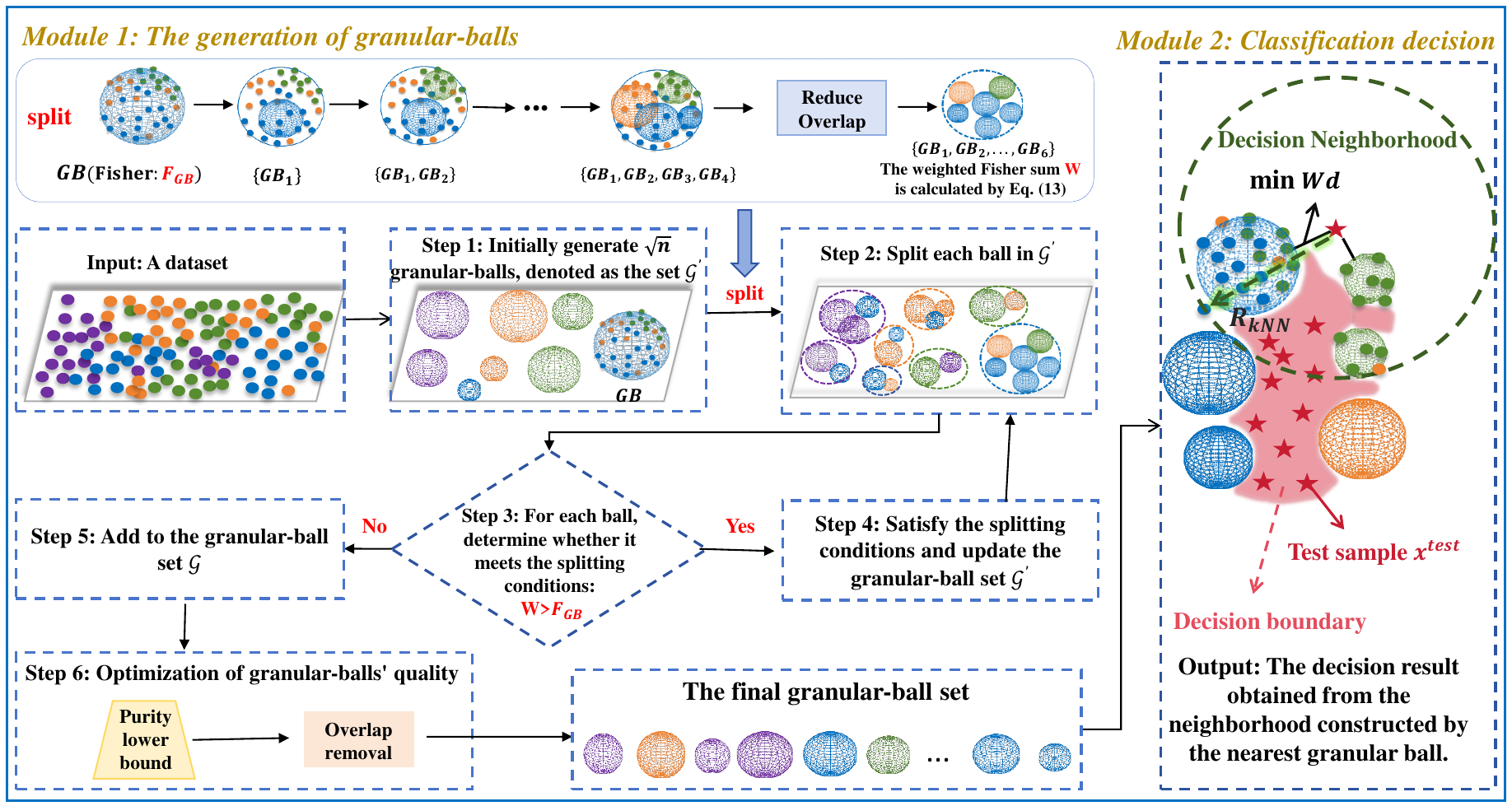}
	\caption{\textcolor{black}{The GBKNN, as illustrated in the overall flowchart, consists of two main modules: i) \textcolor{black}{granular ball} generation; ii) classification decision.}
    }
	\label{gbknn}
\end{figure*}

\section{Proposed method} \label{sec3}
\subsection{Motivation}


\textcolor{black}{Although existing granular ball-based KNN methods~\cite{xia2019granular} have clear advantages over traditional KNN in classification efficiency and robustness, they still have limitations in both granular ball generation and decision making. First, in supervised granular ball generation, existing methods usually rely on specified purity thresholds or purity-related functions as the partition criterion. As a result, they focus mainly on label information and cannot fully capture the geometric distribution of samples within each granular ball. Second, in the prediction stage, these methods typically classify directly based on the nearest ball, lacking a mechanism to construct a local neighborhood around the test sample, and thus failing to reflect the true distribution near the test sample fully. Although MGKNN~\cite{xie2024mgnr} enhances neighborhood modeling through granular ball clustering and enables adaptive adjustment of $k$, its granular ball generation still mainly relies on unsupervised mechanisms, making it difficult to fully exploit class supervision information and limiting its ability to represent true class boundaries. 
To address the above issues, this paper proposes a new supervised and fully adaptive granular ball generation method. The method first obtains an initial granular ball representation through $\sqrt{n}$-based coarse division and then uses the Fisher criterion, which captures both distributional representation and class discrimination, to control the granular ball generation process. The coarse-grained granular ball representation can effectively reduce the influence of noise points, thereby improving model robustness. In the decision stage, an adaptive local neighborhood is constructed for each test sample based on the nearest granular ball, and the value of $k$ is automatically determined by the number of samples in the neighborhood. This not only further reduces the influence of small noisy granular balls but also better matches the decision mechanism of KNN.}

The overall algorithm, as illustrated in Fig.~\ref{gbknn}, consists of two main modules: the \textcolor{black}{granular ball} generation and classification decision.
First, the entire dataset is globally split into~$\sqrt{n}$ initial \textcolor{black}{granular balls} to simplify data distributions within each ball (Fig.~\ref{gbknn}, Step 1). A coarse-to-fine splitting procedure is then performed, guided by \textcolor{black}{the Fisher criterion} that adaptively controls the splits without relying on preset thresholds. 
To further enhance decision boundary representation, a new purity lower bound based on label information is introduced, followed by a global overlap removal strategy to improve the structural clarity of \textcolor{black}{granular balls}.
In the classification decision stage (Fig.~\ref{gbknn}, Module 2), test samples are not granularized. 
\textcolor{black}{Instead, the nearest granular ball is located through a weighted distance mechanism, and the distance from the test sample to the farthest sample within that ball is used as the neighborhood radius $R_{kNN}$ to construct an adaptive neighborhood centered on the test sample, thereby enabling efficient dynamic selection of $k$.}

\subsection{Adaptive \textcolor{black}{Granular Ball} Generation}

\textbf{Initial Partitioning}. In the generation of \textcolor{black}{granular balls}, directly splitting the entire dataset may be affected by the complexity of data distribution, particularly in the presence of nonlinear structures such as crescents or rings. This is because GBG++ determines the center and radius of each \textcolor{black}{granular ball} solely based on majority-class samples. Such geometric assumptions fail to accommodate elongated or curved data regions, making it difficult to align \textcolor{black}{granular balls} with the true shape of the data. As a result, the generated \textcolor{black}{granular balls} often exhibit excessive overlaps or leave out relevant samples, as shown in Fig.~\ref{sqrt(n)}(a) and~\ref{sqrt(n)}(c), ultimately degrading the model’s ability to represent intrinsic structures accurately. 
To better handle complex data structures, the dataset $D = \{ x_i \mid i = 1, 2, \ldots, n \}$ is initially partitioned into $\sqrt{n}$ \textcolor{black}{granular balls} using the $k$-means algorithm.
Assume that the dataset contains $s$ classes, and $D^l$ denotes the subset of samples in the $l$-th class. The initial center $o^{(0)}$ is randomly selected according to the proportion of each class, i.e.,
\begin{equation}\label{sqrt_1}
o_{j}^{(0)} = { \rm Random} \left( \left\{ x_i \in D^l |l=1,2,..,s \right\} \right), j = 1,2,..., \sqrt{n}.
\end{equation}
Each data point $x_i \in D$ is assigned to the nearest centroid $o_j$ using $k$-means clustering:
\begin{equation}\label{sqrt_2}
{GB_{j^{*} }^{(t)} }  = \left\{x_i \in D |  {j^{*} } = arg \min_{j} \left\{ \Delta \left( x_i,o_{j}^{(t)}\right) \right\}\right\},
\end{equation}
where $t$ is the iteration index. The centroid is then updated as:
\begin{equation}\label{sqrt_3}
{o_{j}^{(t+1)} }  = \frac{1}{\left | GB_{j}^{(t)} \right | } \sum_{x_i \in GB_{j}^{t}}  x_i. 
\end{equation}
The alternating iteration follows Eq.~(\ref{sqrt_2}) and Eq.~(\ref{sqrt_3}) until the centroids converge, ultimately yielding $\sqrt{n}$ stable \textcolor{black}{granular balls} and completing the initial data split.

To further validate the effectiveness of first performing~$\sqrt{n}$ splitting, we conducted experiments on the crescent and ring-shaped datasets, and visualized the results using the GBG++ method, as shown in Fig.~\ref{sqrt(n)}(b) and~(d). Compared to directly applying the GBG++ method for splitting, the proposed strategy yields a more accurate approximation of the data distribution and better preserves the intrinsic geometric structure of the data. Specifically, on the crescent and ring-shaped datasets, the GBG++ method generated 109 and 85 \textcolor{black}{granular balls}, respectively, whereas the number was significantly reduced to 21 and 62 when preceded by $\sqrt{n}$-based splitting. This reduction is due to the initial $\sqrt{n}$ splitting, which effectively reduces the global mixing of data points, resulting in more stable subsequent \textcolor{black}{granular ball} partitioning and avoiding the computational overhead and overlap issues caused by excessive splitting. Therefore, this approach not only improves the data fitting capability but also optimizes the \textcolor{black}{granular ball} split efficiency, providing a better solution for applying the \textcolor{black}{granular ball} method to complex data structures.

\textbf{\textcolor{black}{Granular Ball} Splitting Method:} After generating $\sqrt{n}$ initial \textcolor{black}{granular balls}, we further refine the split. As reviewed in Section II-B, the GBG++ method employs an attention-based splitting strategy~\cite{xie2024gbg++}, which demonstrates superior performance in terms of computational efficiency and stability compared to traditional methods such as $k$-means and $k$-partitioning. Specifically, GBG++ effectively focuses on key data regions, enabling more accurate \textcolor{black}{granular ball} splitting while minimizing redundant computations and enhancing overall efficiency. Therefore, we adopt the GBG++ approach described in Eq.~(\ref{GBG++1}),~(\ref{GBG++2}),~(\ref{GBG++22}) and~(\ref{GBG++3}) to further split the initial \textcolor{black}{granular balls}, ensuring structural stability and optimizing computational costs.

In classification tasks, the overlap between heterogeneous \textcolor{black}{granular balls} can lead to misclassification of samples in the overlapping regions, thereby affecting classification accuracy. Traditional methods typically perform deduplication only after the \textcolor{black}{granular ball} generation is complete. In contrast, this paper innovatively integrates the deduplication strategy into the \textcolor{black}{granular ball} generation process by executing deduplication after each split~(as shown in Fig.~\ref{gbknn}). Further splitting is performed to reduce their size for overlapping heterogeneous \textcolor{black}{granular balls}, thereby optimizing subsequent \textcolor{black}{granular ball} splits. This strategy not only actively suppresses the accumulation of overlap, preventing excessive overlap caused by multiple rounds of splitting and improving the stability of \textcolor{black}{granular ball} splitting, but also gradually resolves the overlap issue during the splitting process, reducing the overall computational complexity of deduplication and improving computational efficiency.

\textbf{Stop Condition:} In the process of \textcolor{black}{granular ball} generation, the purity threshold is commonly used to guide the splitting of \textcolor{black}{granular balls}. However, due to the diverse characteristics of different datasets, the optimal value of the purity threshold often varies within a certain range, making it challenging to adaptively determine a fixed value in practice. 
\textcolor{black}{Existing supervised adaptive granular ball generation methods usually determine whether to split a ball according to whether the weighted sum of the purities of the resulting child balls increases~\cite{xia2022efficient}. However, purity only measures the proportion of the dominant label within a granular ball and cannot fully characterize the geometric distribution of samples inside it. To address this limitation, this paper introduces a new criterion to control the granular ball generation process, namely the Fisher criterion, which is defined as follows. Suppose a granular ball $GB$ contains samples from $L$ classes. Let $n_l$ denote the number of samples in the $l$-th class, $\mu_l$ the mean of the $l$-th class, and $\mu$ the overall mean of all samples in the granular ball. Then, the Fisher value of the granular ball is defined as:
	\begin{equation}
   F=\frac{\sum_{l=1}^{L}{n_l \left \| {\mu_l-\mu} \right \| }}{\sum_{l=1}^{L}{n_l S_l}},
	\end{equation}
where $S_l$ represents the intra-class divergence of class l, defined as:
\begin{equation}
S_l=\frac{1}{n_l}\sum_{x_i \in {{GB}_l}}{{{ \left \| {x_i-\mu_l}\right \| }^2}}.
\end{equation}
Here, $GB_l$ represents the set of samples of class $l$.}

\textcolor{black}{Unlike the stopping criterion based on the weighted sum of granular balls' purities, the Fisher criterion is not limited to measuring label consistency within a granular ball. Instead, it evaluates both inter-class separability and intra-class compactness from the perspective of feature distribution. The weighted purity sum is essentially based on label proportions. While it can reflect the degree of class mixing within a granular ball, it cannot adequately characterize the geometric structure of samples in the feature space. In contrast, by jointly modeling between-class scatter and within-class scatter, the Fisher criterion can more comprehensively characterize a granular ball, thereby providing a finer and more robust basis for ball splitting and stopping. A small Fisher value indicates insufficient class separation or excessive within-class dispersion, suggesting that the granular ball remains mixed and should be further split.}   

\textcolor{black}{Since a granular ball with purity equal to 1 is directly treated as a terminal ball, we further adopt a Fisher-based splitting criterion for non-pure granular balls. It should be noted that the Fisher criterion is not equivalent to purity, and therefore, we do not claim that it can theoretically guarantee a monotonic increase in purity. Instead,  it focuses on the change in local structure before and after splitting. Specifically, when the weighted Fisher value of the sub-ball set is higher than that of the parent ball, the split achieves a better balance between inter-class separability and intra-class compactness, thereby improving the representation of the granular ball, and the split is accepted; otherwise, it is rejected, and the parent ball is retained.}


\textbf{Optimization of \textcolor{black}{Granular Balls'} Quality:} \textcolor{black}{During granular ball partitioning, maintaining high purity within each granular ball is a key factor affecting classification performance. Existing splitting criteria may mistakenly accept some low-purity parent balls, which can negatively affect the final classification accuracy. For example, the initial \textcolor{black}{granular ball} has a purity of $P=55/100=0.55$. After splitting, two child balls are obtained with purities $P_1=28/50=0.56$ and $P_2=27/50=0.54$, respectively. In this case, although a split occurs, both child balls remain highly mixed, and neither inter-class separability nor intra-class compactness is substantially improved. As a result, the weighted Fisher value of the child balls may be close to, or even lower than, that of the parent ball, so the Fisher criterion would judge the split as ineffective and stop further splitting. However, because the purity is still low, such granular balls may reduce subsequent classification accuracy and weaken the overall quality of the granular ball representation.} To address the above issue, we propose a new category-adaptive purity lower bound $T_{L_j}$, which is defined as follows:
\begin{equation}\label{T_L}
	T_{L_j} = \frac{ {\textstyle \sum_{i=1}^{N}}\left | GB^*_{i}  \right |  }{\left | L_j \right | },
\end{equation}
where $L_j$ denotes the set of the $j$-th classes in the dataset $D$ and $N$ is the total number of \textcolor{black}{granular balls} set $\mathcal{G}$ generated after splitting.

This strategy allows for a more accurate assessment of the \textcolor{black}{granular ball} quality across different classes. It adaptively adjusts the purity for each class, ensuring that the parent \textcolor{black}{granular ball} stops splitting only after achieving sufficient quality. Additionally, it prevents low-quality \textcolor{black}{granular balls} from terminating prematurely, as shown in the diagram, thereby improving the overall \textcolor{black}{granular balls'} quality. In addition, since the de-overlapping operation is applied only locally within sub-balls during the process, residual overlaps between different classes may persist. Therefore, a final global de-overlapping step is required to eliminate such inter-class overlaps and ensure a clear decision boundary.


\subsection{Classification Decision}
\textcolor{black}{Existing granular ball-based KNN methods usually classify directly according to the distance between the test sample and the nearest boundary granular ball, which can improve classification efficiency and robustness to some extent. However, such methods essentially rely on only a single nearest granular ball for decision making and do not explicitly construct a local neighborhood around the test sample. To address this issue, this paper proposes a novel neighborhood decision-making mechanism based on granular balls. Specifically, it first selects the granular ball whose boundary is closest to the test sample, and then uses the distance from the test sample to the farthest sample within that ball as the neighborhood radius, thereby constructing an adaptive neighborhood centered on the test sample.}
To address this issue, this paper innovatively proposes a sample-number-based weighting method, which dynamically adjusts the confidence of \textcolor{black}{granular balls} to mitigate the impact of noisy \textcolor{black}{granular balls} on classification decisions, thereby enhancing classification robustness. Based on this, the weighted distance is defined as follows:
\begin{definition}
Let $\mathcal{G} = \{GB_i \mid i = 1, 2, ..., N\}$ be a \textcolor{black}{granular ball} set. The weighted distance between the test sample $x^{test}$ and the ball $GB_i \in \mathcal{G}$ is defined as:
\textcolor{black}{
\begin{equation}\label{W_d}
Wd_{i} \left( x^{test}, GB_i \right)= \left( 1-\frac{|GB_i|}{\sum_{j=1}^{N} |GB_j|} \right) \cdot \left( \Delta \left( x^{test},o_{i} \right)   - r_{i}\right),
\end{equation}}
where $|GB_i|$ represents the number of samples contained in the ball $GB_i$. 
\end{definition}
The weighted distance enhances the reliability of \textcolor{black}{granular ball} labels by adjusting the boundary distance based on the proportion of sample counts, prioritizing \textcolor{black}{granular balls} with larger sample sizes and reducing the influence of noise or small \textcolor{black}{granular balls} on decision-making.

\textcolor{black}{After obtaining the weighted nearest granular ball for the test sample, we further use the distance from the test sample to the farthest sample within that ball as the neighborhood radius, thereby constructing an adaptive neighborhood centered on the test sample. Since a granular ball consists of multiple similar samples, it can effectively reduce noise interference and provide more stable and reliable local group information than a single nearest sample point. The neighborhood induced by the nearest ball fully covers the samples within that ball while remaining compact. In other words, it preserves the reliable local information represented by the nearest granular ball without unnecessarily expanding the neighborhood.}

\textcolor{black}{Compared with existing KNN methods based on granular-ball computing, this approach explicitly constructs a strict neighborhood relation around the test sample, thus preserving the basic idea that labels are determined by neighborhoods. Compared with traditional KNN, the neighborhood here is first guided by the coarser granular ball, which gives the method advantages in both efficiency and robustness. Moreover, the actual number of samples contained in the neighborhood naturally determines the corresponding $k$ value, so the neighborhood size no longer depends on manual setting but is adaptively determined by the local data distribution around the test sample. Therefore, the proposed method not only retains the efficiency and robustness of multi-granularity granular ball representation, but also preserves the essential KNN characteristic of making decisions based on local neighborhoods.}

\subsection{Algorithm Procedure}

This paper proposes an improved \textcolor{black}{granular ball} KNN classification method, which consists of the generation of \textcolor{black}{granular balls} and classification decision~(as illustrated in Fig.~\ref{gbknn}). To optimize the initial \textcolor{black}{granular ball} partitioning and reduce the number of iterations, this paper directly splits the entire dataset into $\sqrt{n}$ \textcolor{black}{granular balls} and designs a new evaluation index to select the optimal initial \textcolor{black}{granular ball} division scheme, thereby simplifying the data distribution within each ball. Subsequently, the GBG++ method is employed for fast \textcolor{black}{granular ball} splitting, introducing weighted purity of sub-balls for dynamic control, achieving threshold-free and adaptive splitting. \textcolor{black}{Considering that low-quality granular balls may remain after splitting, this paper further introduces an optimization of the granular ball quality stage, which refines the granular ball set based on a new purity threshold to improve ball quality and provide more reliable granular ball representations for subsequent classification decisions. In the classification stage, the proposed method first identifies the nearest granular ball of the test sample using a weighted distance measure, then constructs an adaptive neighborhood centered on the test sample based on that ball, and finally determines the label by using the sample information within the neighborhood. The detailed procedure is given in Algorithm~\ref{alg:algorithm}.}



 \begin{algorithm}[t!]\footnotesize
	\caption{GBKNN algorithm}
	\label{alg:algorithm}
	\textbf{Input}: Noise training set $D_{noisy}$, test set $D_{test}$;\\    
	\textbf{Output}: The accuracy of GBKNN; 
	\begin{algorithmic}[1] 
		\STATE Select $\sqrt{n}$ initial points randomly according to the label ratio multiple times to split $\sqrt{n}$ balls using $k$-means, and evaluate according to Eq.~(\ref{Sum_T}) to select the optimal \textcolor{black}{granular ball} set $\mathcal{G}_{\sqrt{n}}$.
		\STATE Initial $W = 0$, $ \mathcal{G}_{ball} = \emptyset$, $ \mathcal{G}_{next} = \mathcal{G}_{\sqrt{n}}$;
        \WHILE{$\mathcal{G}_{next}\neq \emptyset$}
        \STATE Select one granular ball $GB_j$ from $\mathcal{G}_{next}$, calculate \textcolor{black}{the Fisher value $W$ of $GB_j$};
		\STATE Split the \textcolor{black}{granular ball} $GB_j$ using the GBG++ method;\\ 
		\textcolor{black}{Calculate the sum of all sub-balls' Fisher value $F_{GB_j}$;}  
		\IF {$ W < F_{GB_j}$}
		\STATE Delete $GB_j$ from $\mathcal{G}_{next}$, de-overlap heterogeneous balls in sub-balls and add the resulting balls to $\mathcal{G}_{next}$;
		\ELSE
		\STATE Add $GB_j$ to $\mathcal{G}_{ball}$ and remove it from $\mathcal{G}_{next}$;
		\ENDIF
		\STATE Return to step 4.
		\ENDWHILE      
		\STATE Get $T_{L_j}$ by the Eq.~(\ref{T_L}), and then use the GBG++ method to segment the balls with purity lower than $T_{L_j}$;
		\STATE De-overlap heterogeneous \textcolor{black}{granular balls};
		\STATE \textcolor{black}{According to Eq.~(\ref{W_d}), the weighted distance $W_{d}$ from each test sample to all granular balls is first computed to identify the nearest granular ball $GB^*$. 
        \STATE Then, the distance from the test sample to the farthest sample within that ball is taken as the radius $R_{kNN}(x)=\max_{x_i \in GB^*} \Delta \left( x,x_i \right)$. An adaptive neighborhood centered on the test sample can be constructed. 
        \STATE It can get the label of the test sample can be obtained based on the neighborhood, and the final accuracy can be obtained.}
	\end{algorithmic}
\end{algorithm}
\subsection{Theoretical Investigation}


To further validate the effectiveness of the initial splitting method, a theorem is presented that proves directly that splitting $\sqrt{n}$ \textcolor{black}{granular balls} results in balls with better density compared to the binary splitting method. Specifically, the direct splitting method better maintains the density balance of the clusters during the division, leading to theoretically superior partition results, with these \textcolor{black}{granular balls} exhibiting a clear advantage in terms of density distribution.
 
\begin{theorem}
Let the dataset be $D = \left\{x_1, x_2, \dots, x_n\right\} \subset \mathbb{R}^d$ and assume it is splitted into $\sqrt{n}$ \textcolor{black}{granular balls} $\{GB_i ,i= 1,2 ,...,\sqrt{n} \}$, where $r_i$ denotes the radius of the $i$-th ball. $\Omega _{GB_i}$ denotes the data space region that the \textcolor{black}{granular ball} $GB_i$ covers. The average density $\bar{\rho}$ is defined as:
\begin{equation}
     \bar{\rho} = \frac{1}{\sqrt{n}} \sum_{i=1}^{\sqrt{n}} \rho_i, \ \rho_i = \frac{1}{\sqrt{n}} \sum_{i=1}^{\sqrt{n}} \frac{n_i}{V_i}, 
\end{equation}
where $n_i =\left | GB_i \right |$ denotes the number of samples in $GB_i$, and $V_i$ represents the volume of $GB_i$. Consider the following two partitioning methods:

 \textbf{Method 1}~(Global Partition): The $\sqrt{n}$ \textcolor{black}{granular balls} are obtained by directly applying the k-means algorithm over the entire dataset;
 
  \textbf{Method 2}~(Recursive Partition): The $\sqrt{n}$ \textcolor{black}{granular balls} are generated via recursive binary splitting using 2-means.

  The following conclusions are satisfied:

 ~(i) If the data are uniformly distributed, then the average densities of samples in granular balls produced by both methods are approximately equal;

 ~(ii) Assuming that the data are drawn from a Lipschitz-continuous probability density function $f(x)$, the average density of sample points in the \textcolor{black}{granular balls} generated by \textbf{Method 1} is higher than that of \textbf{Method 2}, i.e., $\bar{\rho}^{(1)} > \bar{\rho}^{(2)}$.
\end{theorem}

\begin{proof}
   ~(i) Uniform Distribution.

    Assume that the dataset is normalized to lie within the unit hypercube, so that all data points $x_i \sim U([0,1]^d)$ are drawn from a uniform distribution over the unit hypercube. For Method 1, the $k$-means algorithm is directly applied to split the data into $\sqrt{n}$ \textcolor{black}{granular balls}. Since the data are uniformly distributed, each \textcolor{black}{granular ball} covers approximately the same region, with volume denoted by $V$, and the number of samples contained in each ball, $|GB_i| \approx \sqrt{n}$, is roughly equal. In $\mathbb{R}^d$, the volume of a ball~\cite{le2017circumspheres} with radius $r$ is given by 
    \begin{equation}
        V_d(r) = \frac{\pi^{d/2}}{\Gamma(1 + d/2)} {r^d} = \xi_d r^d,
    \end{equation}
    where $\Gamma$ denotes the gamma function and $\xi_d$ is a constant. When $d=1$, thus $\xi_1=2$; hence, $V(r)=2r$, which represents the length of a line segment. When $d=2$, we have $\xi_2=\pi$; therefore, $V(r)=\pi r^2$, which represents the area of a circle.  

    
 The average density of sample points in a \textcolor{black}{granular ball} is 
  \begin{equation}
  \bar{\rho}^{(1)} \approx  \frac{\sqrt{n}}{V_d}.
     \end{equation}

For Method 2, under the assumption of uniform data distribution, the recursive splitting process is performed in a balanced manner, resulting in \textcolor{black}{granular balls} with equal volumes; similarly, we have 
  \begin{equation}
 \bar{\rho}^{(2)} \approx  \frac{\sqrt{n}}{V_d}  .
     \end{equation}
     
In summary, when the samples $x_i \sim U([0,1]^d)$ are drawn from a uniform distribution over the unit hypercube, the \textcolor{black}{granular balls} generated by both methods have similar volumes and sample counts, and thus their densities tend to be approximately equal, i.e.,
  \begin{equation}
  \bar{\rho}^{(1)} \approx \bar{\rho}^{(2)} . 
     \end{equation}
     
(ii) Non-uniform Distributions.

Assume that the data samples are drawn from a probability density function $f(x)$. On each \textcolor{black}{granular ball} $GB_i$, the function $f(x)$ has a bounded gradient and satisfies the Lipschitz condition; that is, there exists a constant $L_i >0$ such that~\cite{sriperumbudur2010hilbert}:
  \begin{equation}
  \forall x,y \in GB_i, \left | f(x) -f(y)\right | \le L_i \left \| x-y \right \|.
     \end{equation}
Since the samples follow the distribution $f(x)$, the expected number of samples in $GB_i$ is approximately 
\begin{equation}
  E\left [ \left | GB_i \right |  \right ] \approx n \int_{\Omega _{GB_i}} f\left( x \right ) dx.
\end{equation}
We rewrite the function $f(x)$ on $GB_i$ as:
     \begin{equation}\label{f}
     f\left( x \right ) = f\left( c_i \right ) + \left ( f\left( x \right ) - f\left ( c_i \right ) \right ).
        \end{equation}
Substitute this approximation Eq.~(\ref{f}) into the integral:
    \begin{equation}\label{f2}
    \begin{split}
    \int_{\Omega _{GB_i}} f\left ( x \right ) dx &= \int_{\Omega _{GB_i}}  \left ( f\left( c_i \right ) + \left( f\left( x \right ) - f\left ( c_i \right ) \right ) \right ) dx \\ &=  f\left( c_i \right )V_i + \int_{\Omega _{GB_i}} \left( f\left( x \right ) - f\left( c_i \right ) \right )dx.
    \end{split}
     \end{equation}
By multiplying by the total number of samples $n$, the expected number of samples can be estimated as:
    \begin{equation}\label{f3}
    \begin{split}
  E\left [ \left | GB_i \right |  \right ] & \approx n \int_{\Omega _{GB_i}} f\left ( x \right ) dx \\  &=
  n f\left ( c_i \right )V_i + n \int_{\Omega _{GB_i}} \left ( f\left ( x \right ) - f\left ( c_i \right ) \right )dx \\  &= n f\left ( c_i \right )V_i  +  \epsilon_i.
    \end{split}
    \end{equation}
Therefore, $\epsilon_i$ represents the error in estimating $E\left [ \left | GB_i \right |  \right ]$—that is, the error caused by using $f\left ( c_i \right )V_i$ instead of the integral $\int_{\Omega _{GB_i}} \left ( f \left ( x \right )  \right )dx$. Since $f(x)$ is Lipschitz continuous on $GB_i$, that is,
\begin{equation}
\left | f\left ( x \right ) - f\left ( c_i \right )  \right | \le L_i  \left \| x-c_i \right \| \le L_i r_i . 
\end{equation}

Then, the corresponding error term can be bounded as follows:
\begin{equation}
 \epsilon_i   \le  n \int_{\Omega _{GB_i}} \left ( L_i r_i \right )dx \le n  L_i r_i V_i. 
\end{equation}
So the density $\rho_{i}$ of samples in \textcolor{black}{granular ball} $GB_i$ can be estimated as follows:
\begin{equation}\label{ro}
\rho_{i} = \frac{ n_i} {V_i} = \frac{E\left [ \left | GB_i \right |  \right ]}{V_i} \approx  n f\left( c_i \right ) + n L_i r_i ,
\end{equation}
where $f(c_i)$ denotes the value of the density function at the center $c_i$ of the ball. \textcolor{black}{For each \textcolor{black}{granular ball}, a larger $f\left( c_i \right )$ and a smaller error term $n L_i r_i$ lead to a higher estimated density.}

Method~1 generates $\sqrt{n}$ \textcolor{black}{granular balls} across the entire data space. The samples within each ball are as concentrated as possible, and $r_i$ is small, which increases the concentration of samples within the clusters, i.e., larger $f(c_i)$. Since k-means tends to generate balls with balanced sizes and small $r_i$, the sample concentration within each ball is higher, leading to smaller error. As a result, according to Eq. (\ref{ro}), the estimated density $\rho_i$ for each \textcolor{black}{granular ball} is larger. Therefore, Method~1 yields \textcolor{black}{granular balls} with higher individual densities $\rho_i$, and consequently achieves a larger average density:
\begin{equation}
\bar{\rho}^{(1)} = \frac{1}{\sqrt{n}} \sum_{i=1}^{\sqrt{n}} \rho_i.
\end{equation}

Method 2 uses recursive 2-means splitting. Since each round of division only focuses on local information, the final centers $c_i$ are no longer globally optimal and may fall into regions with lower density $f(x)$. This can result in \textcolor{black}{granular balls} with larger radii or uneven sample distributions, where $f\left( c_i \right )$ is smaller, $r_i$ is larger, and $L_i$ is unstable, leading to larger error terms. Consequently, the average density $\bar{\rho}^{(2)}$ tends to be lower.
Therefore,
\begin{equation}
\bar{\rho}^{(1)} > \bar{\rho}^{(2)}.
\end{equation}
\end{proof}

\textbf{Remark}: In density estimation, the $nL_i r_i$ term represents an upper bound for the error, not the actual trend of density increasing with radius $ r_i$. This is a consequence of using the Lipschitz continuity to estimate the bias introduced in the integral approximation. As a result, the density approximation takes a form related to this argument. In fact, the density $\rho_i $ does not necessarily increase with the radius $r_i$, especially in non-uniform distributions, where the sample density in regions farther from the center of the ball typically decreases. Therefore, the main term $nf(c_i)$ should be considered the core estimate of the density, while the error term $n L_i r_i$ is only used to quantify the bias and does not reflect the actual trend of density change.

In the above method, the selection of initial points is crucial. Initially, we randomly selected initial points based on the proportion of samples with different labels. However, this random strategy does not guarantee the optimality of \textcolor{black}{granular ball} splitting, often requiring multiple computations and validation to determine the best partitioning, leading to high computational costs. Moreover, random initialization may result in an imbalanced distribution of initial \textcolor{black}{granular balls}, which can affect the stability of subsequent splits and the overall quality of the final \textcolor{black}{granular balls}.

To address the above issue, a heuristic criterion based on density evaluation is proposed to optimize the initial split of $\sqrt{n}$ \textcolor{black}{granular balls}. Specifically, we first generate $\sqrt{n}$ initial \textcolor{black}{granular balls}, denoted as $\mathcal{G}_{\sqrt{n}}$, and introduce the sum of \textcolor{black}{granular balls'} densities as an evaluation metric to assess the quality of the partitioning. This criterion enhances the separability of the initial \textcolor{black}{granular balls} and provides a more stable and effective foundation for subsequent refinement. The mathematical definition is as follows:
\begin{equation}\label{Sum_T}
	Sum_T = \sum_{j=1}^{\sqrt{n}} {\rho}_{j},
\end{equation}
where ${\rho}_{j}$ denotes the density estimate of the \textcolor{black}{granular ball} $GB_j \in \mathcal{G}_{\sqrt{n}}$, reflecting the compactness of the sample distribution within the ball. A higher density indicates that the samples are more tightly clustered, suggesting stronger representational and discriminative capabilities of the \textcolor{black}{granular ball}.
By introducing a density-based evaluation criterion, a relatively optimal initial \textcolor{black}{granular ball} split can be achieved with fewer random selections. This approach effectively reduces uneven distributions caused by randomness, enhancing the stability and reproducibility of the splitting process. Moreover, it promotes a more compact and balanced initial structure, providing a more robust foundation for subsequent \textcolor{black}{granular ball} refinement. As a result, the overall computational cost is reduced, and generation efficiency is significantly improved.

\subsection{Time Complexity}

The overall time complexity of the proposed algorithm can be \textcolor{black}{considered} from two stages: the \textcolor{black}{granular ball} generation stage and the classification decision stage. The \textcolor{black}{granular ball} generation stage primarily consists of the initial partitioning of \textcolor{black}{granular balls} and the dynamic splitting process based on the GBG++ method. \textcolor{black}{In contrast, the classification decision stage uses a weighted nearest ball to construct a neighborhood centered on the predicted point and completes the classification.} Next, we analyze the computational complexity of each stage separately.

\subsubsection{Granular Ball Generation Stage} 

\textcolor{black}{First, we apply $k$-means to obtain an initial division of the dataset into $\sqrt{n}$ granular balls, which simplifies subsequent computations. The time complexity of $k$-means is typically $O(nkT_k)$, where $T_k$ denotes the number of $k$-means iterations. By setting $k=\sqrt{n}$, the time complexity of this initial division stage becomes $O(n\sqrt{n}T_k)$.} Subsequently, during the GBG++ iterative splitting stage, \textcolor{black}{each ball is controlled based on the weighted Fisher sum of its split sub-balls}, ultimately forming $G$ \textcolor{black}{granular balls}
Since the time complexity of GBG++ is close to linear, i.e., $O(n)$, and \textcolor{black}{granular balls} tend to stabilize during the intermediate stages, the overall complexity for \textcolor{black}{granular ball} splitting can be approximated as $O(n)$~\cite{xie2024gbg++}. Furthermore, the de-overlapping operation is performed after each split, but it only affects a small number of adjacent \textcolor{black}{granular balls}, thus the computational load is relatively low and can be neglected. Therefore, the total time complexity for the entire \textcolor{black}{granular ball} generation stage can be expressed as $O(n\sqrt{n}T_k) + O(n) = O(n\sqrt{n})$.

\subsubsection{Classification Decision Stage} 

\textcolor{black}{In the classification stage, for each test sample, the weighted distances to all granular balls are first computed to identify the weighted nearest ball and initialize the neighborhood radius. Let the number of test samples be $M$, and the final number of \textcolor{black}{granular balls} be $N$. The time complexity of this step is $O(MN)$.
To avoid exhaustive search over all training samples, the proposed method further uses the geometric boundaries of granular balls to filter candidate neighborhoods. Specifically, only those granular balls whose boundary distance to the test sample is no greater than the neighborhood radius are considered to potentially intersect the neighborhood. Therefore, sample-level distance computation is performed only within these candidate balls. In this way, neighborhood construction consists of two levels: a coarse filtering at the granular ball level, followed by a fine filtering at the sample level within the candidate balls. Let $s_i$ denote the number of samples in the candidate granular balls of the $i$-th test sample that require further distance computation. Then, the total complexity of the sample-level fine filtering over the whole test set is $O(\sum_{i=1}^M s_i)$. Accordingly, the total time complexity of the classification stage can be written as $O\left ({MN+\sum_{i=1}^{M}{s_i}}\right )$. If $\bar{s}$ denotes the average number of samples that need to be further checked within the candidate granular balls for each test sample, the above complexity can be further simplified as $O\left ({M\left (N+\bar{s}\right )}\right )$.}

\subsubsection{Overall Time Complexity Analysis}

\textcolor{black}{Traditional KNN is a lazy learning method, whose main computational cost is concentrated in the testing stage, with complexity $O(Mn)$. In contrast, the proposed method consists of two stages: a granular ball generation stage and a classification stage. The complexity of the granular ball generation stage is $O(n\sqrt{n})$, while the complexity of the classification stage is $O(M(N+\bar{s}))$. Therefore, the total complexity of the proposed method is $O\!\left(n\sqrt{n} + M(N+\bar{s})\right).$
When the number of test samples is large and $N+\bar{s} \ll n$, the additional cost introduced by preprocessing can be compensated by the efficiency gain in the subsequent classification stage, so that the proposed method achieves an overall computational advantage.}

\section{Experiment details}\label{sec:experiment}
This section conducts an experiment analysis from four main aspects: i) comparing the classification accuracy of different methods across various datasets; ii) analyzing the accuracy variation trends of each method under different noise conditions; iii) evaluating the runtime performance of different methods; \textcolor{black}{iv) evaluating the classification accuracy and running time of GBKNN on large-scale datasets; v) verifying, through ablation studies, the effectiveness of the $\sqrt{n}$-based coarse partition, neighborhood construction, and purity lower-bound mechanism in improving model performance.}


\subsection{Setup}

\textbf{Datasets:} 
\textcolor{black}{To verify the feasibility and effectiveness of the proposed method, experiments were conducted on 17 real-world datasets, including 12 commonly used datasets from the UCI Machine Learning Repository\footnote[1]{\url{http://archive.ics.uci.edu/ml.}} and 5 newly added large-scale datasets (Santander\footnote[2]{\url{https://www.kaggle.com/c/santander-customer-transaction-prediction/overview}}, creditcard\footnote[3]{\url{https://www.kaggle.com/datasets/mlg-ulb/creditcardfraud?utm_source=chatgpt.com}}, covertype\footnote[4]{\url{https://huggingface.co/datasets/inria-soda/tabular-benchmark}}, Higgs\footnotemark[4], and supersymmetry\footnotemark[1]).} The dataset ranges from 169 to \textcolor{black}{1,048,575} samples, with feature dimensions varying between 2 and \textcolor{black}{200} and the number of classes ranging from 2 to 9. Further details are presented in Table~\ref{tab:table 1}.

\begin{table}[!ht]
	\centering
    	\caption{Data information}
	\begin{tabular}{ccccc}
		\midrule[1pt]
		NO. & Datasets           & Sam & Dim & Classes \\ \midrule
		1             & monks-2             & 169          & 7          & 2       \\
		2             & heart1              & 294          & 13         & 2       \\
		3            & cleveland           & 297          & 13         & 5       \\
		4             & haberman            & 306          & 3          & 2       \\
		5             & balance-scale       & 625          & 4          & 2       \\
		6             & fourclass           & 862          & 2          & 2       \\
		7             & phoneme             & 5,404         & 5          & 2       \\
		8            & Anuran              & 7,195         & 22         & 7       \\
		9             & mushroom            & 8,124         & 22         & 2       \\
		10            & anneal              & 20,867        & 10         & 9       \\
		11            & rssi                & 23,500        & 3          & 3       \\
		12            & Skin NonSkin       & 35,961        & 3          & 2       \\ 
        \textcolor{black}{13}            & \textcolor{black}{Santander}    & \textcolor{black}{199999}       & \textcolor{black}{200}        & \textcolor{black}{2}       \\
        \textcolor{black}{14}            & \textcolor{black}{creditcard}    & \textcolor{black}{284806}       & \textcolor{black}{30}        & \textcolor{black}{2}        \\
		\textcolor{black}{15}            & \textcolor{black}{covertype}    & \textcolor{black}{566601}   &  \textcolor{black}{10}    &  \textcolor{black}{2}       \\
        \textcolor{black}{16}            & \textcolor{black}{Higgs}       & \textcolor{black}{937882}      & \textcolor{black}{24}         & \textcolor{black}{2}       \\ 
		\textcolor{black}{17}            & \textcolor{black}{supersymmetry}       & \textcolor{black}{1048575}      & \textcolor{black}{18}         & \textcolor{black}{2}       \\ 
		\midrule
	\end{tabular}
	\label{tab:table 1}
\end{table}

\textbf{Baselines:} The proposed method is compared with several baseline algorithms, including KNN~\cite{soucy2001simple}, NaNEKNN~\cite{zhu2016natural}, adaKNN variants~\cite{mullick2018adaptive}, SMKNN~\cite{ayyad2019gene}, LMKNN~\cite{ayyad2019gene}, OneStepKNN~\cite{zhang2021KNN}, PLKNN~\cite{jodas2022pl}, MGNR~\cite{xie2024mgnr}, \textcolor{black}{GBKNN\_2019~\cite{xia2019granular} and the method based on sub-balls weighted purity and control ball division, and using weighted nearest granular ball decision-making, is denoted as GBKNN\_p.} 

\begin{table*}[!h]
    \centering
    \centering
	\scriptsize
	\setlength
	\tabcolsep{3.5pt} 
    \caption{Average classification accuracy of different methods under various noise conditions}
        \begin{tabular}{lcccccccccccccc}
            \midrule[1pt]

            Data & monks-2 & heart1 & cleveland & haberman & balance-scale & fourclass & phonme & Anuran & mushroom & anneal & rssi & Skin Nonskin & Avg. \\
            \midrule
            KNN  & 0.5798 & 0.6281 & 0.4680 & 0.6684 & 0.8949 & 0.9884 & \textbf{0.8437} & 0.9830 & 0.9876 &\ul { 0.9231} & 0.8585 & 0.9844 & 0.8173 \\ 
            NaNEKNN & \textbf{0.7395} & 0.7044 & 0.4483 & 0.7347 & 0.8903 & 0.9527 & 0.8332 & 0.9526 & 0.9675 & 0.8689 & 0.8477 & \ul {0.9883} & 0.8273 \\ 
            adaKNNGIHS & 0.6807 & 0.6305 & 0.2266 & 0.5408 & 0.8880 & 0.9826 & 0.7846 & 0.8973 & \textbf{0.9935} & 0.7369 & 0.6919 & 0.9848 & 0.7532 \\ 
            adaKNN & 0.6471 & 0.6404 & 0.4483 & 0.5714 & 0.8480 & 0.9701 & 0.7836 & 0.9626 & 0.9900 & 0.8852 & 0.7820 & 0.9874 & 0.7930 \\ 
            adaKNN2 & 0.6555 & 0.6232 & 0.4384 & 0.5893 & 0.8366 & 0.9709 & 0.7995 & 0.9520 & 0.9318 & 0.8834 & 0.8165 & 0.9819 & 0.7899 \\ 
            adaKNN2GIHS & 0.5798 & 0.5788 & 0.2020 & 0.4923 & 0.8274 & 0.9635 & 0.7786 & 0.9042 & 0.9305 & 0.8174 & 0.7515 & 0.9634 & 0.7325 \\ 
            SMKNN & 0.6218 & 0.7586 & 0.5320 & \ul{0.7985} & 0.9211 & 0.8646 & 0.7955 & 0.9435 & 0.9537 & 0.9042 & 0.8280 & 0.8156 & 0.8114 \\ 
            LMKNN & 0.6050 & \ul {0.7759} & 0.4729 & 0.7806 & 0.9223 & 0.8796 & 0.7939 & 0.6941 & 0.9256 & 0.7315 & 0.8524 & 0.7200 & 0.7628 \\ 
            OneStepKNN & 0.4580 & 0.6576 & 0.4483 & 0.7168 & 0.4400 & 0.5905 & 0.4262 & 0.5728 & 0.5029 & 0.7278 & 0.5240 & 0.6802 & 0.5621 \\
            PLKNN & 0.6008 & 0.7389 & 0.4138 & 0.7500 & 0.8754 & 0.7965 & 0.7818 & 0.9237 & 0.9210 & 0.8906 & 0.7387 & 0.6988 & 0.7608 \\ 
            MGNR & 0.7017 & 0.6379 & 0.4483 & 0.6199 & 0.8686 & 0.9718 & 0.8334 & 0.9684 & 0.9382 & 0.9177 & 0.8280 & 0.9714 & 0.8088 \\ 
            \textcolor{black}{GBKNN\_2019}   & \textcolor{black}{0.7227}  & \textcolor{black}{0.7118}  & \textcolor{black}{0.5222}  & \textcolor{black}{0.7168}  & \textcolor{black}{0.8994}  & \textcolor{black}{0.8937}  & \textcolor{black}{0.8115}  & \textcolor{black}{0.9831}  & \textcolor{black}{0.9341}  & \textcolor{black}{0.9168} & \textcolor{black}{0.8043}  & \textcolor{black}{0.9175}  & \textcolor{black}{0.8195}  \\ 
            \textcolor{black}{GBKNN\_p} & 0.7143 & \ul{0.7783} & \ul{0.5369} & 0.7628 & \ul{0.9637} & \ul{0.9950} & 0.8401 & \textbf{0.9869} & \ul{0.9924} & 0.9204 & \ul{0.8884} & \textbf{0.9922} & \ul{0.8622} \\ 
            \textcolor{black}{GBKNN} & \textcolor{black}{\ul {0.7311}} & \textcolor{black}{\textbf{0.8030}} & \textcolor{black}{\textbf{0.5567}}& \textcolor{black}{\textbf{0.8291}} & \textcolor{black}{\textbf{0.9669}} & \textcolor{black}{\textbf{0.9967}} & \textcolor{black}{\ul{0.8403}} & \textcolor{black}{\ul{0.9858}} & \textcolor{black}{0.9896} & \textcolor{black}{\textbf{0.9286}} & \textcolor{black}{\textbf{0.8991}} & \textcolor{black}{0.9766} & \textcolor{black}{\textbf{0.8736}} \\
            \midrule
        \end{tabular}
        \label{acc}
\end{table*}

\textbf{Implementation details:} For the traditional KNN algorithm, the value of $k$ is set within the range $[1, 15]$ with a step size of 2, and the optimal $k$ is determined through hyperparameter tuning using the training set. In the generation of \textcolor{black}{granular balls}, our method first randomly selects the initial point 10 times to split $\sqrt{n}$ balls and then selects the optimal splitting result for subsequent calculations based on the evaluation index. To ensure the comparability of performance evaluation across different algorithms on the same dataset, classification accuracy~(ACC) is adopted as the evaluation metric, which quantifies the consistency between classification results and ground truth labels. A higher ACC value indicates better classification performance. 
In the experiments, all datasets are randomly split into train and test sets in a $8:2$ ratio to ensure fairness and robustness.
Each dataset is augmented by introducing noise at different proportions: 0\%, 5\%, 10\%, 15\%, 20\%, 25\%, and 30\%.
It is important to note that noise is introduced only to the training set, while the validation and test sets remain free of noise, ensuring the reliability and validity of the results. For all experiments, our hardware experimental environment is an AMD Ryzen Threadripper PRO 5975WX 32-Cores 3.60 GHz, and the software experiment environment is Python 3.9.

\subsection{Main Results}
In Table~\ref{acc}, the average accuracy results of different methods under 7 noise levels are presented, with the optimal results highlighted in \textbf{bold} and the second-best results marked with \textbf{\ul{underline}}. As shown in Table~\ref{acc}, the GBKNN algorithm achieves the best or second-best accuracy on most datasets, significantly outperforming the other \textcolor{black}{13} algorithms. Particularly under various noise conditions, \textcolor{black}{GBKNN and GBKNN\_p show notable improvements} in average accuracy compared to the latest multi-granularity-based MGNR method. \textcolor{black}{GBKNN further achieves better performance than GBKNN\_p}. For example, on the `heart1' dataset, the accuracy of the MGNR method is 0.6379, while the proposed GBKNN algorithm achieves an accuracy of \textcolor{black}{0.8030}, surpassing MGNR by approximately \textcolor{black}{25.88\%}. On all datasets, the GBKNN method reaches an average accuracy of \textcolor{black}{0.8736}, outperforming other algorithms. \textcolor{black}{In comparison, the average accuracies of GBKNN\_2019 and GBKNN\_p are 0.8195 and 0.8622, respectively. Therefore, GBKNN improves the average accuracy by approximately 6.60\% and 1.32\% compared with GBKNN\_2019 and GBKNN\_p, respectively.}

\textcolor{black}{This result fully demonstrates the outstanding performance of the GBKNN algorithm in classification tasks, clearly surpassing existing methods. The superior performance of GBKNN mainly comes from its adaptive granular ball representation and neighborhood construction mechanism. Specifically, the initial partition into $\sqrt{n}$ coarse granular balls preserves the overall distribution while simplifying local data structures, thereby reducing the effect of complex nonlinear distributions. Compared with GBKNN\_p, GBKNN uses the Fisher criterion for granular ball partitioning. Unlike purity, the Fisher criterion simultaneously captures inter-class separability and intra-class compactness, thus providing more reasonable guidance for granular ball generation. Classifying based on coarse-grained granular balls helps reduce the impact of noise points and enhances robustness. In the decision stage, GBKNN first searches for the nearest granular ball by a weighted distance, and then uses the distance from the test sample to the farthest sample within that ball as the neighborhood radius to construct an adaptive neighborhood around the test sample. Since a granular ball consists of multiple similar samples, it provides more reliable local group information than a single nearest neighbor. Meanwhile, the neighborhood induced by the nearest granular ball can fully cover the reliable local sample group while remaining compact, which helps to enhance robustness and improve the overall classification accuracy.}

\subsection{Effect of varying noise levels.}
To evaluate the impact of different noise levels, multiple noise intensities were considered, with the results presented in Fig.~\ref{acc_noisy}. Overall, as the noise level increases, most algorithms exhibit varying degrees of performance degradation, highlighting the sensitivity of KNN-based methods to noise interference. Notably, OneStepKNN, adaKNNGHIS, and adaKNN2GHIS show the most significant fluctuations, particularly on datasets such as balance-scale and phoneme, where their accuracy drops substantially, limited in robustness to noisy conditions.


In contrast, GBKNN demonstrates notable stability and robustness. Across all 12 datasets, GBKNN consistently maintains high accuracy under high-noise conditions, with minimal performance decline. For example, on the `haberman', and `rssi' datasets, GBKNN maintains an accuracy above 0.75 even when noise reaches 30\%. \textcolor{black}{On the `heart1' dataset, GBKNN consistently outperforms all other compared methods across all noise levels. 
This notable advantage is mainly attributed to the group-based decision mechanism of GBKNN. Unlike methods that rely on individual sample points, coarse-grained granular balls reduce the impact of noise, thereby providing more reliable and robust local information. Furthermore, the neighborhood radius is defined as the distance between the test sample and the farthest sample within the nearest granular ball, thus constructing an adaptive neighborhood centered at the test sample. Such a neighborhood can fully cover the relevant samples within the nearest granular ball while maintaining good local compactness, thereby effectively enhancing the discriminative stability of the model in noisy environments. This is also a key reason why GBKNN can maintain high accuracy and exhibit strong noise resistance even under high noise levels.}


In comparison, traditional methods such as KNN, LMKNN, and PLKNN perform reasonably well in low-noise scenarios but degrade rapidly as noise increases, reflecting their lack of mechanisms to handle data perturbations. NaNEKNN and SMKNN exhibit performance close to GBKNN on specific datasets such as `mushroom' and `fourclass', but are generally less stable. Additionally, MGNR performs well on datasets like `heart1' and `anneal', suggesting that its natural neighbor strategy is effective in some tasks. However, its performance still suffers under high noise due to susceptibility to false neighbors. In summary, GBKNN exhibits strong adaptability and noise resilience across diverse data distributions and noise levels. It stands out as a leading approach among existing KNN variants in terms of both stability and accuracy, making it particularly suitable for noise-sensitive or structurally complex tasks.

\begin{figure*}[!ht]  
    \centering  
    \begin{subfigure}[b]{0.8\textwidth}
        \centering
        \includegraphics[width=\textwidth]{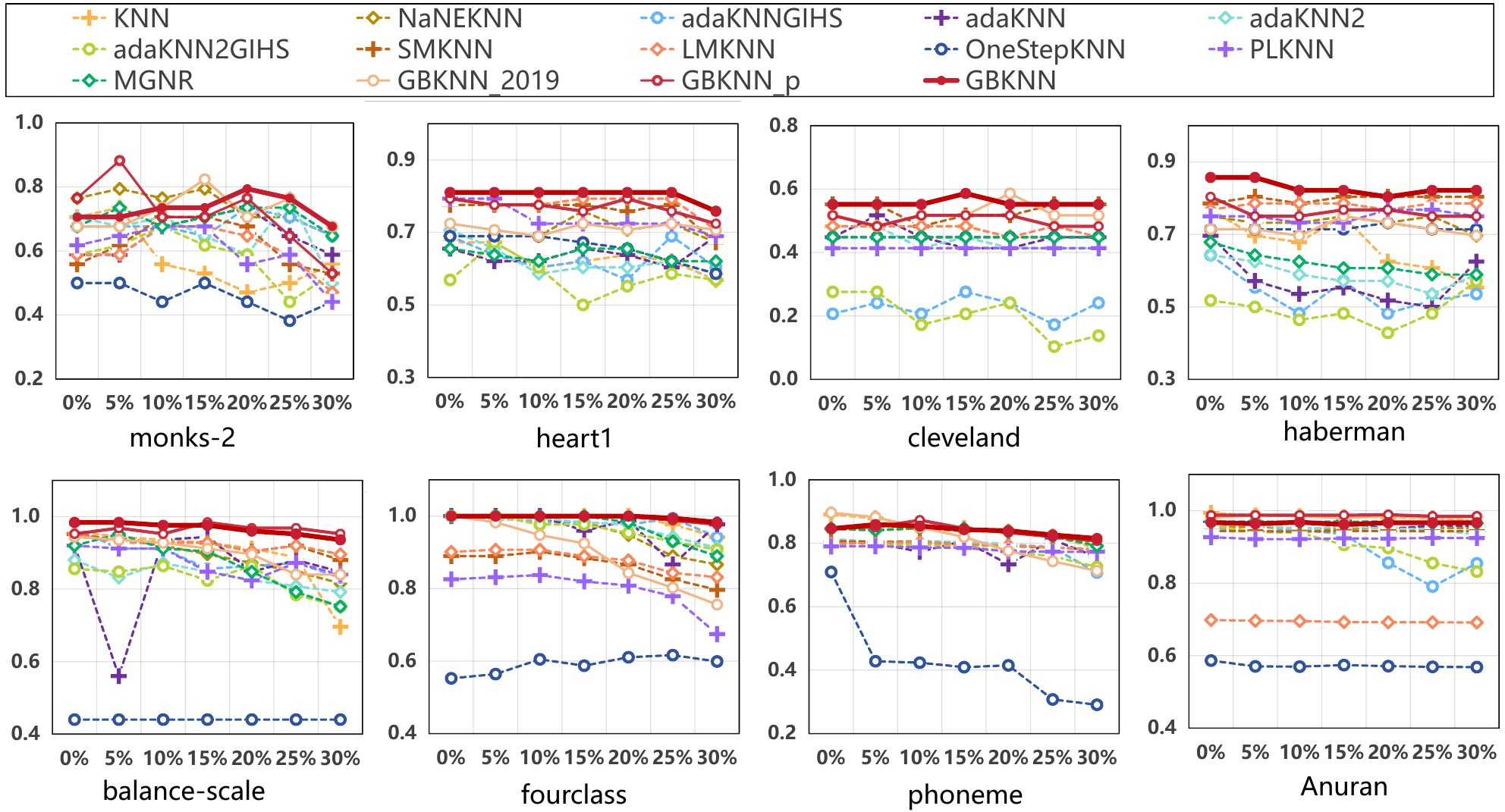}
        \label{fig:image2}
    \end{subfigure}
    \vspace{0.3em} 
    \begin{subfigure}[b]{0.8\textwidth}
        \centering
        \includegraphics[width=\textwidth]{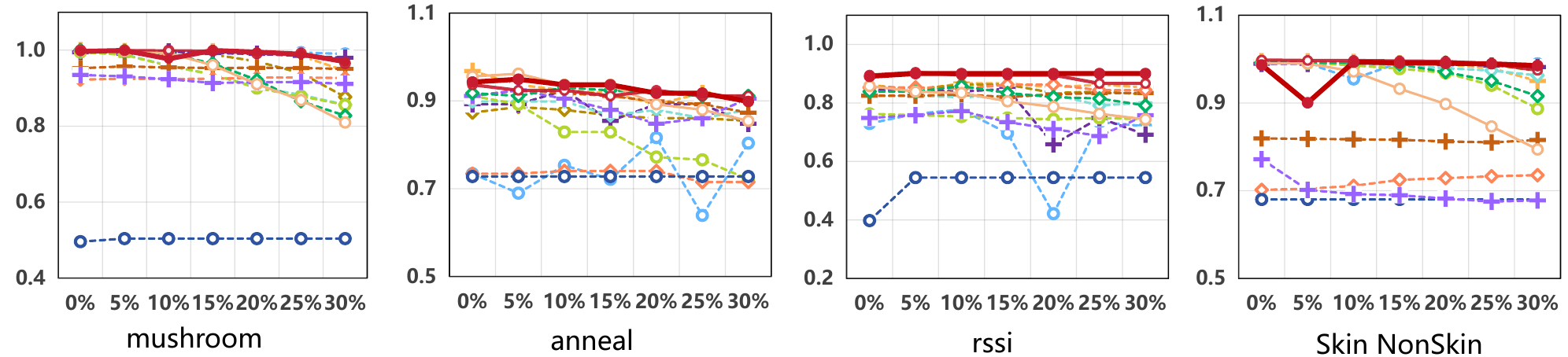}
        \label{fig:image22}
    \end{subfigure}
\caption{\textcolor{black}{Accuracy of different methods on six datasets under different noise conditions.}}  
 \label{acc_noisy}  
\end{figure*}

\begin{table*}[!b]
    \centering
    \centering
    \scriptsize
    \setlength
    \tabcolsep{3.5pt} 
    \caption{Running time (s) of different methods under noise-free conditions.}
    \begin{tabular}{lccccccccccccccc}
        \midrule[1pt]
        Data & monks-2 & heart1 & cleveland & haberman & balance-scale & fourclass & phoneme & Anuran & mushroom & anneal & rssi & Skin Nonskin \\ 
        \midrule
        KNN  & \ul{0.0290} & 0.0630 & 0.0634 & \ul{0.0530} & \textbf{0.1620} & 0.2831 & 9.3909 & 33.9094 & 30.2797 & 0.3682 & 0.3942 & 28.3518 \\ 
        NaNEKNN & 0.0240 & 0.1620 & 0.0810 & 0.0806 & \textbf{0.1620} & 0.2066 & 10.2117 & 48.7547 & 14.8327 & 0.6892 & 0.4001 & 25.5310 \\
        adaKNNGIHS & 0.3621 & 0.6977 & 0.5754 & 0.6663 & 2.4360 & 4.2957 & 149.0276 & 269.1356 & 346.7465 & 3.7422 & 6.4180 & 455.3578 \\ 
        adaKNN & 0.3375 & 0.7073 & 0.5743 & 0.6635 & 2.4263 & 4.3463 & 150.0538 & 271.5230 & 350.7985 & 3.7151 & 6.3709 & 453.9935 \\ 
        adaKNN2 & 0.1422 & 0.3383 & 0.3327 & 0.3172 & 1.2322 & 2.3371 & 74.4621 & 133.2459 & 175.4169 & 1.9179 & 3.2745 & 223.5811 \\
        adaKNN2GIHS & 0.1400 & 0.3355 & 0.3329 & 0.3169 & 1.2068 & 2.1702 & 75.1372 & 132.3127 & 173.2620 & 1.8857 & 3.2012 & 221.6835 \\
        SMKNN & 0.0150 & \textbf{0.0450} & \textbf{0.0236} & \textbf{0.0420} & 0.1950 & 0.5251 & 17.0071 & 32.0827 & 31.3878 & 0.3066 & 0.5468 & 42.4478 \\ 
        LMKNN & \textbf{0.0160} & \ul{0.0470} & \ul{0.0240} & \textbf{0.0420} & 0.2017 & 0.6069 & 17.4207 & 34.4481 & 32.1668 & 0.3418 & 0.6053 & 44.9805 \\ 
        OneStepKNN & 0.1097 & 0.1210 & 0.0581 & 0.0702 & 0.2921 & 0.4803 & 70.4204 & 156.3371 & 224.2986 & 0.3792 & 0.9113 & 334.7534 \\ 
        PLKNN & 0.0290 & 0.0746 & 0.0360 & 0.0820 & 0.3993 & 0.9231 & 31.0830 & 45.8409 & 55.3659 & 0.4587 & 0.9540 & 79.5953 \\ 
        MGNR & 0.0420 & 0.0630 & 0.0460 & 0.0540 & \ul{0.1513} & 0.1960 & 3.8934 & 7.9492 & 6.5127 & \textbf{0.2326} & \ul{0.2831} & 7.7279 \\ 
        \textcolor{black}{GBKNN\_2019} & \textcolor{black}{3.0944} & \textcolor{black}{4.7226} & \textcolor{black}{5.9680} & \textcolor{black}{5.1024} & \textcolor{black}{5.7159} & \textcolor{black}{1.3050} & \textcolor{black}{41.9653} & \textcolor{black}{8.8376} & \textcolor{black}{2.7996} & \textcolor{black}{4.4277} & \textcolor{black}{8.2658} & \textcolor{black}{3.7593} \\ 
        \textcolor{black}{GBKNN\_p} & 0.0990 & 0.1140 & 0.1210 & 0.1066 & 0.2160 & \textbf{0.1860} & \textbf{2.1095} & \textbf{2.0994} & \textbf{2.0175} & \ul{0.2560} & \textbf{0.2826} & \textbf{2.8242} \\ 
        \textcolor{black}{GBKNN} & \textcolor{black}{0.1130} & \textcolor{black}{0.1370} & \textcolor{black}{0.1320} & \textcolor{black}{0.1270} & \textcolor{black}{0.2560} & \textcolor{black}{\ul{0.2310}} & \textcolor{black}{\ul{2.4616}} & \textcolor{black}{\ul{4.3782}} & \textcolor{black}{\ul{2.4396}} & \textcolor{black}{0.3676} & \textcolor{black}{0.3490} & \textcolor{black}{\ul{3.2114}} \\ 
        \midrule
    \end{tabular}
    \label{time} 
\end{table*}

\subsection{Runtime Analysis}
 
\textcolor{black}{As shown in Table~\ref{time}, the runtime comparison of various algorithms under noise-free conditions is conducted on 12 datasets. Overall, traditional KNN demonstrates relatively high computational efficiency on some small-scale datasets, such as `heart1', `cleveland', and `haberman', where it achieves relatively short runtimes. Meanwhile, several improved KNN-based methods also exhibit certain efficiency advantages on some datasets. In contrast, GBKNN\_p and GBKNN incur slightly higher runtimes on small-scale datasets, mainly due to the additional granular ball generation process before classification, which introduces extra preprocessing overhead. 
However, as the data scale increases, the efficiency advantage of granular ball-based methods becomes more evident. On datasets such as `phoneme', `Anuran', `mushroom', `anneal', and `Skin NonSkin', both GBKNN\_p and GBKNN achieve significantly lower runtimes than most competing algorithms. For example, on the `anneal' dataset, the runtime of GBKNN is only 0.3576, which is much faster than 4.2327 for GBKNN\_2019 and 30.2779 for the traditional KNN. Similarly, on `phoneme', `mushroom', and `Skin NonSkin', GBKNN achieves runtimes of 2.4616, 2.4396, and 3.2114, respectively, demonstrating strong computational efficiency. Overall, the granular ball-based representation effectively reduces the computational cost caused by sample-wise operations, leading to better scalability on medium and large scale datasets.}

\textcolor{black}{Further comparison among GBKNN\_2019, GBKNN\_p, and GBKNN shows that GBKNN\_p is generally slightly faster than GBKNN, while GBKNN consistently outperforms GBKNN\_2019 in terms of runtime. This difference mainly stems from their distinct granular ball generation and decision mechanisms. Specifically, GBKNN\_2019 improves efficiency by replacing sample points with granular balls, but its splitting strategy is relatively inefficient. GBKNN\_p adopts a coarse division and controls granular ball splitting based on purity, and directly performs classification using the nearest granular ball, resulting in a simpler computational process. In contrast, GBKNN employs the Fisher criterion to guide granular ball splitting and stopping, enabling more refined partitioning. In the prediction stage, it constructs an adaptive neighborhood centered at the test sample based on the nearest granular ball rather than directly using it for classification. Although this neighborhood construction introduces additional computational cost, it also enhances classification performance.}

\textcolor{black}{In summary, the results in Table III indicate that both GBKNN\_p and GBKNN achieve favorable time efficiency while maintaining high classification performance. GBKNN\_p is more efficient in terms of computational speed, whereas GBKNN achieves a better balance between efficiency and discriminative power, demonstrating strong practical value.}

\subsection{\textcolor{black}{Results of large-scale datasets}}

\textcolor{black}{Furthermore, to further validate the applicability of the proposed method in large-scale scenarios, experiments are conducted on five large-scale datasets, including `Santander', `creditcard', `covertype', `Higgs', and `supersymmetry', and the results are reported in the table. Due to the large sample sizes and high computational cost of these datasets, some competing algorithms encounter issues such as excessive memory consumption or prohibitively long runtimes, making it difficult to complete the experiments within acceptable resource and time constraints. Therefore, complete comparison results are not available on these datasets. In this setting, the analysis primarily focuses on the classification performance and robustness of GBKNN on large-scale datasets.}

\begin{table}[!h]
\centering
 \centering
 \scriptsize
 \setlength
    \tabcolsep{3.5pt} 
\caption{\textcolor{black}{Accuracy and time performance of GBKNN on large-scale datasets at different noise levels.}}
  {\color{black}
\begin{tabular}{llllll}
\midrule[1pt]
Data    & Santander  & creditcard & covertype  & Higgs      & supersymmetry \\ \hline
0\%     & 0.8979     & 0.9991     & 0.8183     & 0.5117     & 0.6973        \\
5\%     & 0.8979     & 0.9989     & 0.8179     & 0.5139     & 0.6948        \\
10\%    & 0.8979     & 0.9990     & 0.8170     & 0.5119     & 0.6920        \\
15\%    & 0.8979     & 0.9991     & 0.8126     & 0.5144     & 0.6932        \\
20\%    & 0.8979     & 0.9987     & 0.8153     & 0.5122     & 0.6926        \\
25\%    & 0.8979     & 0.9986     & 0.8053     & 0.5220     & 0.6921        \\
30\%    & 0.8979     & 0.9980     & 0.7973     & 0.5197     & 0.6857        \\ \hline
Time(s) & 4183.4577 & 7367.7777  & 11644.1096 & 17343.2118 & 17271.1339    \\ \hline
\end{tabular}}
 \label{bigdata}
\end{table}

\textcolor{black}{As shown in Table~\ref{bigdata}, GBKNN maintains relatively stable performance under different noise levels. For example, on the `Santander' dataset, the accuracy remains at 0.8979 across the noise range from 0\% to 30\%; on the `creditcard' dataset, the accuracy consistently stays above 0.9980; and on the `Higgs' dataset, the accuracy remains stable within the range of 0.5117–0.5220 with only minor fluctuations. For the `covertype' and `supersymmetry' datasets, although accuracy declines as noise level increases, it remains stable overall.
These results indicate that GBKNN retains strong robustness and applicability when dealing with large-scale and high-complexity data. Taken together with the average runtime of all noise observations, it can be seen that although the algorithm still requires a certain amount of computation time for ultra-large datasets, it can complete the training and prediction processes under limited computational resources, demonstrating good scalability.
Overall, these findings further verify the effectiveness of GBKNN in large-scale data scenarios and highlight the advantage of granular ball representation in reducing computational burden and improving algorithm feasibility for large datasets.}

\subsection{Ablation Studies}


\begin{table*}[!h]
    \centering
    \centering
    \scriptsize
    \setlength
    \tabcolsep{1.1pt} 
    \caption{\textcolor{black}{Comparison of classification accuracy of GBKNN under different key modules.}}
    {\color{black}
    \begin{tabular}{cccccccccccccccccccc}
        \midrule[1pt]
$\sqrt{n}$    &  Purity                    & fisher                    & minball      & minball\_$W_d$      & $T_L$        & Neighborhood     & monks-2         & heart1          & cleveland       & haberman        & balance-scale   & fourclass       & phoneme         & Anuran          & mushroom        & anneal          & rssi            & Skin Nonskin    & Ave.            \\ \hline
\Checkmark &   \Checkmark &                           & \Checkmark &  & \Checkmark   &                       & 0.4832          & 0.6798          & 0.4483          & 0.7423          & 0.8697          & 0.9759          & 0.8180          & \textbf{0.9798} & 0.9916 & 0.8915          & 0.8429          & \textbf{0.9880} & 0.8093          \\
\Checkmark &  \Checkmark &       &               & \Checkmark   & \Checkmark    &        & 0.4916          & 0.6847          & 0.4581          & 0.7372          & 0.8697          & 0.9759          & 0.8182          & 0.9797 & \textbf{0.9917} & 0.8870          & 0.8470          & 0.9877 & 0.8107          \\
 \Checkmark  & &  \Checkmark &  &  \Checkmark   & \Checkmark  &    & 0.5             & 0.7266          & \textbf{0.4778} & 0.7066          & 0.8754          & 0.9801          & 0.8103          & 0.9780          & 0.9897          & 0.8915          & 0.8490          & 0.9877 & 0.8144          \\
   & & \Checkmark    & & \Checkmark   & \Checkmark    & \Checkmark        & 0.6933          &  0.6059 & 0.4532          & 0.6301 & 0.8229 & 0.9352 & 0.8057 & 0.9666          & 0.9607          & \textbf{0.9105}          & 0.7949 & 0.9211          & 0.7917 \\ 
     \Checkmark & & \Checkmark    & & \Checkmark   &     & \Checkmark        & 0.5294          & \textbf{0.7463}    & 0.4532          & 0.7321   & 0.9234        & 0.9784     & 0.8248     & 0.9589          & 0.9829          & 0.8716          & 0.8558     & 0.9828          & 0.8200 \\   
 \Checkmark   & & \Checkmark    & & \Checkmark   & \Checkmark    & \Checkmark        & 0.5336     & 0.7340      & 0.4581    & \textbf{0.7526}       & \textbf{0.9246}       & \textbf{0.9809}       & \textbf{0.8263}       & 0.9436          & 0.9863          & 0.8761          & \textbf{0.8626} & 0.9820          & \textbf{0.8217} \\ \hline
    \end{tabular}}
    \label{example} 
\end{table*}

\textcolor{black}{As shown in Table~\ref{example}, all key designs improve the final performance, and the complete model integrating $\sqrt{n}$-based coarse initialization, the Fisher criterion, the lower bound $T_L$ of quality, weighted boundary distance $W_d$, and neighborhood decision achieves the best average result of 0.8217. To ensure a fair comparison in the ablation study, we fixed the $\sqrt{n}$ initial cluster centers and analyzed the effect of different designs on model performance based on this setting.}
Subsequently, according to the proportion of different class samples, the farthest sample from the already selected initial points was iteratively chosen until $\sqrt{n}$ initial points were selected. By prioritizing the selection of majority-class samples and incorporating a farthest-point strategy based on global distribution, the initial \textcolor{black}{granular balls} were ensured to cover the entire data space, leading to a more balanced \textcolor{black}{granular ball} division and reducing initial partitioning imbalance.

\textcolor{black}{In the ablation study on the $\sqrt{n}$-based coarse division, we compared the traditional granular ball initialization method, which starts the iterative process from the whole dataset, with the $\sqrt{n}$-based coarse initialization strategy. The results are shown in Table~\ref{example} (Rows 4 and 6). It can be seen that introducing coarse division leads to higher classification accuracy on most datasets, with improvements exceeding 5\%–10\% on some datasets. The reason is that dividing the whole dataset into $\sqrt{n}$ coarse balls provides a simpler initial representation for the original complex distribution, which helps reduce the effect of nonlinear data distributions and avoids premature convergence to overly large granular balls, thereby providing a more reasonable basis for adaptive granular ball generation.}

\textcolor{black}{The ablation experiment results for Purity and Fisher control of granular ball generation are shown in Table ~\ref{example} (Rows 2 and 3). The results show that replacing purity with the Fisher criterion as the control criterion for granular ball generation further improves the average performance of the model by 3.7\%. This indicates that, compared with purity, which only considers the proportion of the dominant class, the Fisher criterion can simultaneously capture inter-class separability and intra-class compactness within a granular ball, thus enabling more accurate evaluation of ball quality and more effective guidance for adaptive granular ball division.}

\textcolor{black}{As shown in Table~\ref{example} (Rows 5 and 6), after introducing the quality lower-bound mechanism (i.e., $T_L$), the classification accuracy improves on most datasets. In particular, the accuracy increases by 2.05\% and 1.32\% on the `haberman' and `rssi' datasets, respectively. This mechanism further refines those granular balls that do not satisfy the splitting condition but still have relatively low quality, thereby effectively improving the classification accuracy of GBKNN across different datasets.} 

\textcolor{black}{The effectiveness of the weighted boundary distance can be verified by comparing Rows 1 and 2 in Table~\ref{example}. Under the same purity-based granular ball generation criterion, the average performance improves from 0.8093 to 0.8107. Unlike the unweighted version, which treats all granular balls equally during prediction, the weighted boundary distance reduces the influence of small noisy balls and enhances the role of larger balls in nearest-ball search. Since larger granular balls usually contain more representative samples and are less likely to be dominated by local noise, this mechanism improves the reliability of nearest-ball selection and thus enhances the robustness of the model in the presence of local noise and outliers.}

\textcolor{black}{As shown in Rows 3 and 6 of Table~\ref{example}, the decision strategy based on constructing a neighborhood around the test point outperforms direct decision making based on the nearest granular ball. Since this neighborhood is induced by the nearest granular ball, and its radius is defined as the distance from the test sample to the farthest sample within that ball, it can fully cover the reliable local sample group while remaining compact. More importantly, this mechanism allows $k$ to be adaptively determined by the actual number of samples contained in the neighborhood, rather than being manually preset, which is more consistent with the basic KNN principle of making decisions based on local neighborhoods.}

\section{Conclusion}\label{sec6}

\textcolor{black}{This paper proposes a new adaptive granular-ball k-nearest neighbor model. The method fits the original data distribution by generating granular balls. It first simplifies local splitting through $\sqrt{n}$ coarse divisions, and then achieves full adaptivity by controlling the splitting process with the weighted sum of the Fisher values of child balls. During coarse-granular ball partitioning, a density-based splitting criterion is introduced, and the $\sqrt{n}$ divisions strategy is optimized through rigorous mathematical analysis. In the decision stage, GBKNN constructs an adaptive neighborhood for each test sample based on the weighted nearest granular ball, so that the value of $k$ is dynamically determined by the actual number of samples in the neighborhood, avoiding the dependence of traditional KNN on a fixed parameter setting. Extensive experiments on multiple benchmark datasets show that GBKNN outperforms traditional KNN and other competing methods in classification accuracy, time efficiency, and noise robustness, demonstrating strong overall advantages. In future work, we will further explore more efficient granular ball partition strategies to continuously improve classification performance. We will also extend GBKNN to broader application scenarios, such as imbalanced classification, open-set recognition, and semi-supervised learning.}


%

\bibliographystyle{IEEEtran}
\bibliography{reference}

\end{document}